\documentclass[11pt]{article}
\usepackage[colorlinks=true,bookmarks=false,linkcolor=blue,urlcolor=blue,citecolor=blue,breaklinks=true]{hyperref}
\usepackage[lmargin=3cm,rmargin=3cm,bottom=3cm,top=3cm]{geometry}

\usepackage[utf8]{inputenc} % allow utf-8 input
\usepackage[T1]{fontenc}    % use 8-bit T1 fonts
\usepackage{hyperref}       % hyperlinks
\usepackage{url}            % simple URL typesetting
\usepackage{booktabs}       % professional-quality tables
\usepackage{amsfonts}       % blackboard math symbols
\usepackage{nicefrac}       % compact symbols for 1/2, etc.
\usepackage{microtype}      % microtypography
\usepackage{subfigure}
\usepackage[round]{natbib}

\usepackage{amsmath,amssymb,amsthm}
\usepackage{mathtools}

\usepackage{algorithm}
\usepackage[algo2e, ruled, linesnumbered]{algorithm2e}
\usepackage{multirow}
\usepackage{verbatim}

\DeclareMathOperator*{\E}{\mathbb{E}}
\let\Pr\relax
\DeclareMathOperator*{\Pr}{\mathbb{P}}

\newcommand{\eps}{\epsilon}
\newcommand{\inprod}[1]{\left\langle #1 \right\rangle}
\newcommand{\R}{\mathbb{R}}
\newcommand{\D}{\mathcal{D}}
\newcommand{\W}{\mathcal{W}}
\newcommand{\Loss}{\mathcal{L}}

\DeclarePairedDelimiter{\norm}{\lVert}{\rVert}

\DeclareMathOperator*{\argmin}{arg\,min}

\newtheorem{theorem}{Theorem}
\newtheorem{corollary}{Corollary}
\newtheorem{lemma}{Lemma}

\newtheorem{definition}{Definition}

\newtheorem{fact}{Fact}

\newtheorem*{app-lemma}{Lemma}
\newtheorem*{app-theorem}{Theorem}
\newtheorem*{app-definition}{Definition}

\title{Sparse Logistic Regression Learns All Discrete Pairwise Graphical Models}

\author{
Shanshan Wu, Sujay Sanghavi, Alexandros G.~Dimakis\\
\texttt{shanshan@utexas.edu}, \texttt{sanghavi@mail.utexas.edu}, \\
\texttt{dimakis@austin.utexas.edu}\\
\\
{\it Department of Electrical and Computer Engineering}\\
{\it University of Texas at Austin}\\
}

\date{}

\begin{document}

\maketitle

\begin{abstract}
We characterize the effectiveness of a classical algorithm for recovering the Markov graph of a general discrete pairwise graphical model from i.i.d. samples. The algorithm is (appropriately regularized) maximum conditional log-likelihood, which involves solving a convex program for each node; for Ising models this is $\ell_1$-constrained logistic regression, while for more general alphabets an $\ell_{2,1}$ group-norm constraint needs to be used. We show that this algorithm can recover any arbitrary discrete pairwise graphical model, and also characterize its sample complexity as a function of model width, alphabet size, edge parameter accuracy, and the number of variables. We show that along every one of these axes, it matches or improves on all existing results and algorithms for this problem. Our analysis applies a sharp generalization error bound for logistic regression when the weight vector has an $\ell_1$ constraint (or $\ell_{2,1}$ constraint) and the sample vector has an $\ell_{\infty}$ constraint (or $\ell_{2, \infty}$ constraint). We also show that the proposed convex programs can be efficiently solved in $\tilde{O}(n^2)$ running time (where $n$ is the number of variables) under the same statistical guarantees. We provide experimental results to support our analysis.
\end{abstract}

\section{Introduction}

Undirected graphical models provide a framework for modeling 
high dimensional distributions with dependent variables and have many applications 
including in computer vision~\citep{MLTW10}, bio-informatics~\citep{MCK12}, and sociology~\citep{EPL09}. 
In this paper we characterize the effectiveness of a natural, and already popular, algorithm for the \emph{structure learning} problem. Structure learning is the task of finding the dependency graph 
of a Markov random field (MRF) given i.i.d. samples; typically one is also interested in finding estimates for the edge weights as well. We consider the structure learning problem in general (non-binary) discrete pairwise graphical models. These are MRFs where the variables take values in a discrete alphabet, but all interactions are pairwise. This includes the Ising model as a special case (which corresponds to a binary alphabet).

The natural and popular algorithm we consider is (appropriately regularized) maximum conditional log-likelihood for finding the neighborhood set of any given node. For the Ising model, this becomes $\ell_1$-constrained logistic regression; more generally for non-binary graphical models the regularizer becomes an $\ell_{2,1}$ norm. We show that this algorithm can recover all discrete pairwise graphical models, and characterize its sample complexity as a function of the parameters of interest: model width, alphabet size, edge parameter accuracy, and the number of variables. We match or improve dependence on each of these parameters, over all existing results for the general alphabet case when no additional assumptions are made on the model (see Table~\ref{table:comps_non_binary}). For the specific case of Ising models, some recent work has better dependence on some parameters (see Table~\ref{table:comps} in Appendix~\ref{sec:comps}). 

We now describe the related work, and then outline our contributions.

\subsection*{Related Work}
In a classic paper, \citet{RWL10} considered the structure learning problem for Ising models. They showed that $\ell_1$-regularized logistic regression provably recovers the correct dependency graph with a very small number of samples by solving a convex program for each variable. This algorithm was later generalized to multi-class logistic regression with group-sparse regularization, which can learn MRFs with higher-order interactions and non-binary variables~\citep{JRVS11}. A well-known limitation of~\citep{RWL10, JRVS11} is that their theoretical guarantees only work for a restricted class of models. Specifically, they require that the underlying learned model satisfies technical~\textit{incoherence} assumptions, that are difficult to validate or check.

{\renewcommand{\arraystretch}{1.1}%
\begin{table} [htbp]
\begin{center}
 \begin{tabular}{|p{3.1cm}|p{5.5cm}|p{5.0cm}|} 
 \hline
Paper  & Assumptions & Sample complexity ($N$) \\
\hline
\multirow{4}{3.1cm}{Greedy algorithm~\citep{HKM17}} & 1. Alphabet size $k \ge 2$ & \multirow{4}{5.0cm}{$O(\exp(\frac{k^{O(d)}\exp(O(d^2\lambda))}{\eta^{O(1)}})\ln(\frac{nk}{\rho}))$}\\
    & 2. Model width $\le \lambda$  & \\
	& 3. Degree $\le d$ &\\
	& 4. Minimum edge weight $\ge \eta>0$ &\\
	& 5. Probability of success $\ge 1-\rho$ &\\
 \hline
 \multirow{4}{3.1cm}{Sparsitron~\citep{KM17}} & 1. Alphabet size $k \ge 2$ & \multirow{4}{5.0cm}{$O(\frac{\lambda^2k^5\exp(14\lambda)}{\eta^4}\ln(\frac{nk}{\rho\eta}))$} \\
    &2. Model width $\le \lambda$  &  \\
	&3. Minimum edge weight $\ge\eta>0$ &\\
	&4. Probability of success $\ge 1-\rho$ &\\
\hline
	\multirow{3}{3.1cm}{$\ell_{2,1}$-constrained logistic regression [{\bf this paper}]} & 1. Alphabet size $k \ge 2$  & \multirow{3}{5.0cm}{$O(\frac{\lambda^2k^4\exp(14\lambda)}{\eta^4}\ln(\frac{nk}{\rho}))$}\\
	&2. Model width $\le \lambda$ & \\
	&3. Minimum edge weight $\ge\eta>0$ &\\ 
	&4. Probability of success $\ge 1-\rho$ &\\
\hline 
\end{tabular}
\end{center}

\caption[]{Sample complexity comparison for different graph recovery algorithms. The pairwise graphical model has alphabet size $k$. For $k=2$ (i.e., Ising models), our algorithm reduces to the $\ell_{1}$-constrained logistic regression (see Table~\ref{table:comps} in Appendix~\ref{sec:comps} for related work on learning Ising models). Our sample complexity has a better dependency on the alphabet size ($\tilde{O}(k^4)$ versus $\tilde{O}(k^5)$) than that in~\citep{KM17}\footnotemark.}
\label{table:comps_non_binary}

\end{table}}
\footnotetext{Theorem 8.4 in~\citep{KM17} has a typo. The correct dependence should be $k^5$ instead of $k^3$. In Section 8 of~\citep{KM17}, after re-writing the conditional distribution as a sigmoid function, the weight vector $w$ is a vector of length $(n-1)k+1$. Their derivation uses an incorrect bound $\norm{w}_1 \le 2\lambda$, while it should be $\norm{w}_1 \le 2k\lambda$. This gives rise to an additional $k^2$ factor on the final sample complexity.}

A large amount of recent work has since proposed various algorithms to obtain provable learning results for general graphical models without requiring the incoherence assumptions. We now describe the (most related part of the extensive) related work, followed by our results and comparisons (see Table~\ref{table:comps_non_binary}). For a discrete pairwise graphical model, let $n$ be the number of variables and $k$ be the alphabet size; define the model width $\lambda$ as the maximum neighborhood weight (see Definition~\ref{ising_model} and~\ref{pairwise_model} for the precise definition). 
%For the case of $k=2$ (i.e., Ising models), \citet{SW12} provided an information-theoretic lower bound on the number of samples, which depends logarithmically on $n$, and exponentially on the width $\lambda$.
For structure learning algorithms, a popular approach is to focus on the sub-problem of finding the neighborhood of a single node. Once this is correctly learned, the overall graph structure is a simple union bound. Indeed all the papers we now discuss are of this type. As shown in Table~\ref{table:comps_non_binary}, \citet{HKM17} proposed a greedy algorithm to learn pairwise (and higher-order) MRFs with general alphabet. Their algorithm generalizes the approach of \citet{Bre15} for learning Ising models. The sample complexity in~\citep{HKM17} grows logarithmically in $n$, but {\em doubly} exponentially in the width $\lambda$. Note that an information-theoretic lower bound for learning Ising models~\citep{SW12} only has a {\em single-exponential} dependence on $\lambda$. \citet{KM17} provided a different algorithmic and theoretical approach by setting this up as an online learning problem and leveraging results from the Hedge algorithm therein. Their algorithm Sparsitron achieves single-exponential dependence on the width $\lambda$. 

\subsection*{Our Contributions}

\begin{itemize}
\item Our main result: We show that the $\ell_{2,1}$-constrained\footnote{It may be possible to prove a similar result for the {\em regularized} version of the optimization problem using techniques from~\citep{NRWY12}. One needs to prove that the objective function satisfies restricted strong convexity (RSC) when the samples are from a graphical model distribution~\citep{VMLC16, LVMC18}. It is interesting to see if the proof presented in this paper is related to the RSC condition.} logistic regression can be used to estimate the edge weights of a discrete pairwise graphical model from i.i.d. samples (see Theorem~\ref{pairwise_theorem}). For the special case of Ising models (see Theorem~\ref{ising_theorem}), this reduces to an $\ell_{1}$-constrained logistic regression. We make no incoherence assumption on the graphical models. As shown in Table~\ref{table:comps_non_binary}, our sample complexity scales as $\tilde{O}(k^4)$, which improves the previous best result with $\tilde{O}(k^5)$ dependency\footnote{This improvement essentially comes from the fact that we are using an $\ell_{2,1}$ norm constraint instead of an $\ell_{1}$ norm constraint for learning general (i.e., non-binary) pairwise graphical models (see our remark after Theorem~\ref{pairwise_theorem}). The Sparsitron algorithm proposed by~\citet{KM17} learns a $\ell_{1}$-constrained generalized linear model. This $\ell_{1}$-constraint gives rise to a $k^5$ dependency for learning non-binary pairwise graphical models.}. The analysis applies a sharp generalization error bound for logistic regression when the weight vector has an $\ell_{2,1}$ constraint (or $\ell_1$ constraint) and the sample vector has an $\ell_{2, \infty}$ constraint (or $\ell_{\infty}$ constraint) (see Lemma~\ref{gen_bound} and Lemma~\ref{l21_gen_bound} in Appendix~\ref{sec:proof_kl_lemma}). Our key insight is that a generalization bound can be used to control the squared distance between the predicted and true logistic functions (see Lemma~\ref{KL_lemma} and Lemma~\ref{pairwise_KL_lemma} in Section~\ref{sec_lemmas}), which then implies an $\ell_{\infty}$ norm bound between the weight vectors (see Lemma~\ref{infty_norm} and Lemma~\ref{pairwise_infty_norm}).

\item We show that the proposed algorithms can run in $\tilde{O}(n^2)$ time without affecting the statistical guarantees (see Section~\ref{sec:runtime}). Note that $\tilde{O}(n^2)$ is an efficient runtime for graph recovery over $n$ nodes. Previous algorithms in~\citep{HKM17, KM17} also require $\tilde{O}(n^2)$ runtime for structure learning of pairwise graphical models. 

\item We construct examples that violate the incoherence condition proposed in~\citep{RWL10} (see Figure~\ref{incoherence_example}). We then run $\ell_1$-constrained logistic regression and show that it can recover the graph structure as long as given enough samples. This verifies our analysis and shows that our conditions for graph recovery are weaker than those in~\citep{RWL10}.

\item We empirically compare the proposed algorithm with the Sparsitron algorithm in~\citep{KM17} over different alphabet sizes, and show that our algorithm needs fewer samples for graph recovery (see Figure~\ref{non_binary_comp}).
\end{itemize}

{\bf Notation.} 
We use $[n]$ to denote the set $\{1,2,\cdots, n\}$. For a vector $x\in\R^n$, we use $x_i$ or $x(i)$ to denote its $i$-th coordinate. The $\ell_p$ norm of a vector is defined as $\norm{x}_p=(\sum_i |x_i|^p)^{1/p}$. We use $x_{-i}\in\R^{n-1}$ to denote the vector after deleting the $i$-th coordinate. For a matrix $A\in \R^{n\times k}$, we use $A_{ij}$ or $A(i,j)$ to denote its $(i,j)$-th entry. We use $A(i, :)\in\R^k$ and $A(:,j)\in\R^n$ to the denote the $i$-th row vector and the $j$-th column vector. The $\ell_{p,q}$ norm of a matrix $A\in \R^{n\times k}$ is defined as $\norm{A}_{p,q} = \norm{[\norm{A(1,:)}_p, ..., \norm{A(n,:)}_p]}_q$. We define $\norm{A}_{\infty}=\max_{ij}|A(i,j)|$ throughout the paper (note that this definition is different from the induced matrix norm). We use $\sigma(z) = 1/(1+e^{-z})$ to represent the sigmoid function. We use $\inprod{\cdot, \cdot}$ to represent the dot product between two vectors $\inprod{x,y}=\sum_i x_i y_i$ or two matrices $\inprod{A,B}=\sum_{ij}A(i,j)B(i,j)$.

\section{Main results} \label{sec_main}

We start with the special case of binary variables (i.e., Ising models), and then move to the general case with non-binary variables.

\subsection{Learning Ising models}

We first give a definition of an Ising model distribution.
\begin{definition}
Let $A\in\R^{n\times n}$ be a symmetric weight matrix with $A_{ii} = 0$ for $i\in[n]$. Let $\theta\in\R^n$ be a mean-field vector. The $n$-variable Ising model is a distribution $\D(A, \theta)$ on $\{-1,1\}^n$ that satisfies
\begin{equation}
\Pr_{Z\sim\D(A, \theta)}[Z = z]\propto \exp(\sum_{1\le i<j\le n} A_{ij}z_i z_j + \sum_{i\in[n]}\theta_iz_i).
\end{equation}
The dependency graph of $\D(A,\theta)$ is an undirected graph $G=(V,E)$, with vertices $V = [n]$ and edges $E=\{(i,j): A_{ij}\neq 0\}$. 
The width of $\D(A,\theta)$ is defined as
\begin{equation}
\lambda(A,\theta) = \max_{i\in[n]}(\sum_{j\in[n]}|A_{ij}| + |\theta_i|). 
\end{equation}
Let $\eta(A,\theta)$ be the minimum edge weight in absolute value, i.e., $\eta(A, \theta) = \min_{i,j\in[n]: A_{ij}\neq 0} |A_{ij}|$.
\label{ising_model}
\end{definition}

One property of an Ising model distribution is that the conditional distribution of any variable given the rest variables follows a logistic function. Let $\sigma(z) = 1/(1+e^{-z})$ be the sigmoid function.

\begin{fact}
Let $Z\sim \D(A, \theta)$ and $Z\in \{-1,1\}^n$. For any $i\in[n]$, the conditional probability of the $i$-th variable $Z_i\in \{-1,1\}$ given the states of all other variables $Z_{-i} \in \{-1,1\}^{n-1}$ is
\begin{equation}
\Pr[Z_i = 1 | Z_{-i} = x] = \frac{\exp(\sum_{j\neq i} A_{ij}x_j + \theta_i)}{\exp(\sum_{j\neq i} A_{ij}x_j + \theta_i)+\exp(-\sum_{j\neq i} A_{ij}x_j - \theta_i)} = \sigma(\inprod{w,x'}),
\end{equation}
where $x' = [x, 1]\in\{-1,1\}^n$, and $w = 2[A_{i1},\cdots, A_{i(i-1)}, A_{i(i+1)},\cdots, A_{in},\theta_i]\in\R^{n}$. Moreover, $w$ satisfies $\norm{w}_1\le 2\lambda(A,\theta)$, where $\lambda(A, \theta)$ is the model width defined in Definition~\ref{ising_model}.
\label{logistic_dist}
\end{fact}

Following Fact~\ref{logistic_dist}, the natural approach to estimating the edge weights $A_{ij}$ is to solve a logistic regression problem for each variable. 
For ease of notation, let us focus on the $n$-th variable (the algorithm directly applies to the rest variables). Given $N$ i.i.d. samples $\{z^1, \cdots, z^N\}$, where $z^i\in\{-1,1\}^n$ from an Ising model $\D(A,\theta)$, we first transform the samples into $\{(x^i, y^i)\}_{i=1}^N$, where $x^i = [z^i_{1},\cdots, z^i_{n-1},1]\in\{-1,1\}^n$ and $y^i = z^i_{n}\in\{-1,1\}$. By Fact~\ref{logistic_dist}, we know that $\Pr[y^i=1|x^i=x] = \sigma(\inprod{w^*,x})$ where $w^* = 2[A_{n1},\cdots, A_{n(n-1)},\theta_n]\in\R^{n}$ satisfies $\norm{w^*}_1\le 2\lambda(A,\theta)$. Suppose that $\lambda(A,\theta)\le \lambda$, we are then interested in recovering $w^*$ by solving the following $\ell_1$-constrained logistic regression problem
\begin{equation}
\hat{w} \in \argmin_{w\in\R^n} \frac{1}{N} \sum_{i=1}^N \ell(y^i\inprod{w,x^i}) \quad
\text{\quad s.t. }\norm{w}_1\le 2\lambda, \label{ERM}
\end{equation}
where $\ell:\R\to\R$ is the loss function
\begin{equation}
\ell(y^i\inprod{w,x^i}) = \ln(1+e^{-y^i\inprod{w,x^i}}) = 
\begin{cases}
-\ln\sigma(\inprod{w,x^i}), & \text{ if } y^i=1 \\
-\ln(1-\sigma(\inprod{w,x^i})), & \text{ if } y^i=-1 
\end{cases} \label{logistic_loss}
\end{equation}
Eq. (\ref{logistic_loss}) is essentially the negative log-likelihood of observing $y^i$ given $x^i$ at the current $w$. 

Let $\hat{w}$ be a minimizer of (\ref{ERM}). It is worth noting that in the high-dimensional regime ($N<n$), $\hat{w}$ may not be unique. In this case, we will show that \emph{any} one of them would work. After solving the convex problem in (\ref{ERM}), the edge weight is estimated as $\hat{A}_{nj} = \hat{w}_j/2$. 

The pseudocode of the above algorithm is given in Algorithm~\ref{ising_learning_algo}. Solving the $\ell_1$-constrained logistic regression problem will give us an estimator of the true edge weight. We then form the graph by keeping the edge that has estimated weight larger than $\eta/2$ (in absolute value).

\begin{algorithm2e}
\caption{Learning an Ising model via $\ell_1$-constrained logistic regression}
\label{ising_learning_algo}
\DontPrintSemicolon
\LinesNumbered
\KwIn{$N$ i.i.d. samples $\{z^1, \cdots, z^N\}$, where $z^m\in\{-1,1\}^n$ for $m\in [N]$; an upper bound on $\lambda(A,\theta)\le \lambda$; a lower bound on $\eta(A, \theta)\ge\eta>0$.}
\KwOut{$\hat{A}\in \R^{n\times n}$, and an undirected graph $\hat{G}$ on $n$ nodes.}
\For{$i\leftarrow 1$ \KwTo $n$}{
%\tcp{Form samples $(x^m, y^m)\in\{-1,1\}^n \times \{-1,1\}$.}
$\forall m\in [N]$, $x^m \leftarrow [z^m_{-i},1]$, $y^m \leftarrow z^m_{i}$\;
%\tcp{Solve an $\ell_1$-constrained logistic regression.}
$\hat{w} \leftarrow \argmin_{w\in\R^n} \frac{1}{N} \sum_{m=1}^N \ln(1+e^{-y^m\inprod{w,x^m}}) \text{\quad s.t. }\norm{w}_1\le 2\lambda$\;
%\tcp{Estimate the edge weight for node $i$.}
$\forall j\in [n]$, $\hat{A}_{ij} \leftarrow  \hat{w}_{\tilde{j}}/2$, where $\tilde{j}=j$ if $j<i$ and $\tilde{j}=j-1$ if $j>i$\;
}
Form an undirected graph $\hat{G}$ on $n$ nodes with edges $\{(i,j): |\hat{A}_{ij}|\ge\eta/2, i<j\}$.
\end{algorithm2e}

\begin{theorem}
Let $\D(A, \theta)$ be an unknown $n$-variable Ising model distribution with dependency graph $G$. Suppose that the $\D(A, \theta)$ has width $\lambda(A, \theta)\le \lambda$. Given $\rho\in(0,1)$ and $\eps>0$, if the number of i.i.d. samples satisfies $N = O(\lambda^2\exp(12\lambda)\ln(n/\rho)/\eps^4)$,
then with probability at least $1-\rho$, Algorithm~\ref{ising_learning_algo} produces $\hat{A}$ that satisfies
\begin{equation}
\max_{i,j\in[n]} |A_{ij} - \hat{A}_{ij}| \le \eps. \label{close_A}
\end{equation}
\label{ising_theorem}
\end{theorem}

\begin{corollary}
In the setup of Theorem~\ref{ising_theorem}, suppose that the Ising model distribution $\D(A, \theta)$ has minimum edge weight $\eta(A, \theta)\ge \eta>0$. If we set $\eps<\eta/2$ in (\ref{close_A}), which corresponds to sample complexity $N = O(\lambda^2\exp(12\lambda)\ln(n/\rho)/\eta^4)$,
then with probability at least $1-\rho$, Algorithm~\ref{ising_learning_algo} recovers the dependency graph, i.e., $\hat{G}=G$.
\end{corollary}

\subsection{Learning pairwise graphical models over general alphabet}
%We first give a definition of the general pairwise graphical model.

\begin{definition}
Let $k$ be the alphabet size. Let $\W = \{W_{ij}\in\R^{k\times k}: i\neq j \in [n]\}$ be a set of weight matrices satisfying $W_{ij} = W_{ji}^T$. Without loss of generality, we assume that every row (and column) vector of $W_{ij}$ has zero mean. Let $\Theta=\{\theta_i\in\R^k: i\in[n]\}$ be a set of external field vectors. Then the $n$-variable pairwise graphical model $\D(\W, \Theta)$ is a distribution over $[k]^n$ where 
\begin{equation}
\Pr_{Z\sim \D(\W, \Theta)}[Z=z] \propto \exp(\sum_{1\le i<j \le n} W_{ij}(z_i, z_j) + \sum_{i\in[n]} \theta_i(z_i)).
\end{equation}
The dependency graph of $\D(\W, \Theta)$ is an undirected graph $G=(V,E)$, with vertices $V = [n]$ and edges $E=\{(i,j): W_{ij}\neq 0\}$. The width of $\D(\W, \Theta)$ is defined as
\begin{equation}
\lambda(\W,\Theta) = \max_{i, a}(\sum_{j\neq i} \max_{b\in[k]}|W_{ij}(a, b)| + |\theta_i(a)|). 
\end{equation}
We define
$\eta(\W,\Theta) = \min_{(i,j)\in E} \max_{a, b} |W_{ij}(a, b)|$.
\label{pairwise_model}
\end{definition}

{\bf Remark.} The assumption that $W_{ij}$ has centered rows and columns (i.e., $\sum_b W_{ij}(a, b) =0$ and $\sum_a W_{ij}(a, b) =0$ for any $a,b\in [k]$) is without loss of generality (see Fact 8.2 in~\citep{KM17}). If the $a$-th row of $W_{ij}$ is not centered, i.e., $\sum_b W_{ij}(a, b) \neq 0$, we can define $W'_{ij}(a, b) = W_{ij}(a,b) - \sum_b W_{ij}(a, b)/k$ and $\theta'_i(a)=\theta_i(a)+\sum_b W_{ij}(a, b)/k$, and notice that $\D(\W, \Theta)=\D(\W', \Theta')$. Because the sets of matrices with centered rows and columns (i.e., $\{M\in\R^{k\times k}: \sum_b M(a, b) =0, \forall a\in[k]\}$ and $\{M\in\R^{k \times k}: \sum_a M(a, b) =0, \forall b\in[k]\}$) are two linear subspaces, alternatively projecting $W_{ij}$ onto the two sets will converge to the intersection of the two subspaces~\citep{Von49}. As a result, the condition of centered rows and columns is necessary for recovering the underlying weight matrices, since otherwise different parameters can give the same distribution. Note that in the case of $k=2$, Definition~\ref{pairwise_model} is the same as Definition~\ref{ising_model} for Ising models. To see their connection, simply define $W_{ij}\in\R^{2\times 2}$ as follows: $W_{ij}(1,1) = W_{ij}(2,2) = A_{ij}$, $W_{ij}(1,2) = W_{ij}(2,1) = -A_{ij}$.

For a pairwise graphical model distribution $\D(\W, \Theta)$, the conditional distribution of any variable (when restricted to a pair of values) given all the other variables follows a logistic function, as shown in Fact~\ref{pairwise_logistic_dist}. This is analogous to Fact~\ref{logistic_dist} for the Ising model distribution.

\begin{fact}\label{pairwise_logistic_dist}
Let $Z\sim \D(\W, \Theta)$ and $Z\in [k]^n$. For any $i\in[n]$, any $\alpha \neq\beta\in [k]$, and any $x\in[k]^{n-1}$,
\begin{equation}
\Pr[Z_i = \alpha | Z_i \in \{\alpha, \beta\}, Z_{-i} = x] = \sigma(\sum_{j\neq i} (W_{ij}(\alpha, x_j)-W_{ij}(\beta, x_j)) + \theta_i(\alpha)- \theta_i(\beta)).
\end{equation}
\end{fact}

Given $N$ i.i.d. samples $\{z^1, \cdots, z^N\}$, where $z^m\in [k]^n\sim \D(\W, \Theta)$ for $m\in[N]$, the goal is to estimate matrices $W_{ij}$ for all $i\neq j\in [n]$. For ease of notation and without loss of generality, let us consider the $n$-th variable. Now the goal is to estimate matrices $W_{nj}$ for all $j\in [n-1]$. 

To use Fact~\ref{pairwise_logistic_dist}, fix a pair of values $\alpha\neq\beta\in [k]$, let $S$ be the set of samples satisfying $z_n\in \{\alpha, \beta\}$. We next transform the samples in $S$ to $\{(x^t, y^t)\}_{t=1}^{|S|}$ as follows: $x^t = \text{OneHotEncode}([z_{-n}^t, 1]) \in \{0,1\}^{n\times k}$, $y^t = 1$ if $z_n^t = \alpha$, and $y^t = -1$ if $z_n^t = \beta$. Here $\text{OneHotEncode}(\cdot):[k]^n\to \{0,1\}^{n\times k}$ is a function that maps a value $t \in [k]$ to the standard basis vector $e_t\in\{0,1\}^k$, where $e_t$ has a single 1 at the $t$-th entry. For each sample $(x,y)$ in the set $S$, Fact~\ref{pairwise_logistic_dist} implies that $\Pr[y=1|x] = \sigma(\inprod{w^*, x})$, where $w^*\in\R^{n \times k}$ satisfies 
\begin{equation}
w^*(j, :) = W_{nj}(\alpha, :) - W_{nj}(\beta, :), \forall j \in [n-1]; \quad w^*(n,:) = [\theta_n(\alpha) - \theta_n(\beta), 0, ..., 0]. \label{w_mat_def}
\end{equation}
Suppose that the width of $\D(\W, \Theta)$ satisfies $\lambda(\W, \Theta)\le \lambda$, then $w^*$ defined in (\ref{w_mat_def}) satisfies $\norm{w^*}_{2,1} \le 2\lambda\sqrt{k}$, where $\norm{w^*}_{2,1}:= \sum_j \norm{w^*(j,:)}_2$. We can now form an $\ell_{2,1}$-constrained logistic regression over the samples in $S$: 
\begin{equation}
w^{\alpha, \beta} \in \argmin_{w\in\R^{n\times k}} \frac{1}{|S|} \sum_{t=1}^{|S|} \ln(1+ e^{-y^t\inprod{w,x^t}}) \quad \text{\quad s.t. }\norm{w}_{2,1}\le 2\lambda\sqrt{k}, \label{pairwise_ERM}
\end{equation}

Let $w^{\alpha, \beta}$ be a minimizer of (\ref{pairwise_ERM}). Without loss of generality, we can assume that the first $n-1$ rows of $w^{\alpha, \beta}$ are centered, i.e., $\sum_{a}w^{\alpha, \beta}(j,a)=0$ for $j\in[n-1]$. Otherwise, we can always define a new matrix $U^{\alpha, \beta}\in\R^{n \times k}$ by centering the first $n-1$ rows of $w^{\alpha, \beta}$:
\begin{align}
U^{\alpha, \beta}(j,b) &= w^{\alpha,\beta}(j,b) - \frac{1}{k}\sum_{a\in[k]}w^{\alpha,\beta}(j, a), \; \forall j\in[n-1],\; \forall b\in[k]; \label{center_w}\\
U^{\alpha, \beta}(n,b) &= w^{\alpha,\beta}(n,b) + \frac{1}{k}\sum_{j\in[n-1], a\in [k]}w^{\alpha,\beta}(j, a), \;\forall b\in[k].  \nonumber
\end{align}
Since each row of the $x$ matrix in (\ref{pairwise_ERM}) is a standard basis vector (i.e., all zeros except a single one), $\inprod{U^{\alpha, \beta}, x} = \inprod{w^{\alpha,\beta}, x}$, which implies that $U^{\alpha, \beta}$ is also a minimizer of (\ref{pairwise_ERM}).

The key step in our proof is to show that given enough samples, the obtained $U^{\alpha, \beta}\in\R^{n \times k}$ matrix is close to $w^*$ defined in (\ref{w_mat_def}). Specifically, we will prove that
\begin{equation}
|W_{nj}(\alpha, b) - W_{nj}(\beta, b) - U^{\alpha, \beta}(j, b)| \le \eps, \quad \forall j\in[n-1],\; \forall \alpha, \beta, b\in[k]. \label{u_diff_w}
\end{equation}
Recall that our goal is to estimate the original matrices $W_{nj}$ for all $j\in [n-1]$. Summing (\ref{u_diff_w}) over $\beta \in [k]$ (suppose $U^{\alpha, \alpha}=0$) and using the fact that $\sum_{\beta} W_{nj}(\beta,b) = 0$ gives
\begin{equation}
|W_{nj}(\alpha, b)-\frac{1}{k}\sum_{\beta\in [k]}U^{\alpha, \beta}(j, b)| \le \eps, \quad \forall j\in[n-1],\; \forall \alpha, b\in [k].
\end{equation}
In other words, $\hat{W}_{nj}(\alpha, :)= \sum_{\beta\in [k]}U^{\alpha, \beta}(j, :)/k$ is a good estimate of $W_{nj}(\alpha, :)$. 

Suppose that $\eta(\W,\Theta)\ge \eta$, once we obtain the estimates $\hat{W}_{ij}$, the last step is to form a graph by keeping the edge $(i,j)$ that satisfies $\max_{a,b}|\hat{W}_{ij}(a, b)| \ge \eta/2$. The pseudocode of the above algorithm is given in Algorithm~\ref{pairwise_learning_algo}. 

\begin{algorithm2e}
\caption{Learning a pairwise graphical model via $\ell_{2,1}$-constrained logistic regression}
\label{pairwise_learning_algo}
\DontPrintSemicolon
\LinesNumbered
\KwIn{alphabet size $k$; $N$ i.i.d. samples $\{z^1, \cdots, z^N\}$, where $z^m\in[k]^n$ for $m\in[N]$; an upper bound on $\lambda(\W,\Theta)\le \lambda$; a lower bound on $\eta(\W, \Theta)\ge\eta>0$.}
\KwOut{$\hat{W}_{ij}\in \R^{k\times k}$ for all $i\neq j \in [n]$; an undirected graph $\hat{G}$ on $n$ nodes.}
\For{$i\leftarrow 1$ \KwTo $n$}{
 \For{each pair $\alpha \neq \beta \in [k]$}{
  $S \leftarrow \{z^m, m\in[N]: z^m_i \in \{\alpha, \beta\}\}$\;
  %\tcp{Form samples $(x^t, y^t)\in\{0,1\}^{n\times k} \times \{-1,1\}$.}
  $\forall z^t\in S$, $x^t \leftarrow \text{OneHotEncode}([z^t_{-i},1])$,  $y^t\leftarrow 1$ if $z^t_i=\alpha$; $y^t\leftarrow -1$ if $z^t_i=\beta$\;
  %\tcp{Solve an $\ell_{2,1}$-constrained logistic regression.}
  $w^{\alpha,\beta} \leftarrow \arg \min_{w\in\R^{n\times k}}  \frac{1}{|S|} \sum_{t=1}^{|S|} \ln(1+e^{-y^t\inprod{w,x^t}}) \text{\quad s.t. }\norm{w}_{2,1}\le 2\lambda\sqrt{k}$\;
  Define $U^{\alpha, \beta}\in\R^{n\times k}$ by centering the first $n-1$ rows of $w^{\alpha,\beta}$ (see (\ref{center_w})).\;
 }
 %\tcp{Estimate the weight matrices for node $i$.}
 \For{$j\in [n]\backslash i$ and $\alpha\in[k]$}{
  $\hat{W}_{ij}(\alpha, :) \leftarrow \frac{1}{k}\sum_{\beta\in [k]} U^{\alpha,\beta}(\tilde{j}, :)$, where $\tilde{j} = j$ if $j< i$ and $\tilde{j} = j-1$ if $j>i$.\;
 }
}
Form graph $\hat{G}$ on $n$ nodes with edges $\{(i,j): \max_{a,b}|\hat{W}_{ij}(a, b)| \ge \eta/2, i<j\}$.
\end{algorithm2e}

\begin{theorem}
Let $\D(\W, \Theta)$ be an $n$-variable pairwise graphical model distribution with width $\lambda(\W, \Theta)\le \lambda$. Given $\rho\in(0,1)$ and $\eps>0$, if the number of i.i.d. samples satisfies $N = O(\lambda^2k^4\exp(14\lambda)\ln(nk/\rho)/\eps^4)$, then with probability at least $1-\rho$, Algorithm~\ref{pairwise_learning_algo} produces $\hat{W}_{ij}\in \R^{k\times k}$ that satisfies
\begin{equation}
|W_{ij}(a,b) - \hat{W}_{ij}(a, b)| \le \eps, \quad \forall i \neq j \in[n], \; \forall a, b\in[k]. \label{close_W}
\end{equation}
\label{pairwise_theorem}
\end{theorem}

\begin{corollary}
In the setup of Theorem~\ref{pairwise_theorem}, suppose that the pairwise graphical model distribution $\D(\W, \Theta)$ satisfies $\eta(\W, \Theta)\ge \eta>0$. If we set $\eps<\eta/2$ in (\ref{close_W}), which corresponds to sample complexity $N = O(\lambda^2k^4\exp(14\lambda)\ln(nk/\rho)/\eta^4)$,
then with probability at least $1-\rho$, Algorithm~\ref{pairwise_learning_algo} recovers the dependency graph, i.e., $\hat{G}=G$.
\end{corollary}

{\bf Remark ($\ell_{2,1}$ versus $\ell_1$ constraint).} The $w^*\in \R^{n \times k}$ matrix defined in (\ref{w_mat_def}) satisfies $\norm{w^*}_{\infty,1}\le 2\lambda(\W,\Theta)$. This implies that $\norm{w^*}_{2, 1}\le 2\lambda(\W,\Theta)\sqrt{k}$ and $\norm{w^*}_{1} \le 2\lambda(\W,\Theta)k$. Instead of solving the $\ell_{2,1}$-constrained logistic regression defined in (\ref{pairwise_ERM}), we could solve an $\ell_1$-constrained logistic regression with $\norm{w}_{1} \le 2\lambda(\W,\Theta)k$. However, this will lead to a sample complexity that scales as $\tilde{O}(k^5)$, which is worse than the $\tilde{O}(k^4)$ sample complexity achieved by the $\ell_{2,1}$-constrained logistic regression. The reason why we use the $\ell_{2,1}$ constraint instead of the tighter $\ell_{\infty,1}$ constraint in the algorithm is because our proof relies on a sharp generalization bound for $\ell_{2,1}$-constrained logistic regression (see Lemma~\ref{l21_gen_bound} in the appendix). It is unclear whether a similar generalization bound exists for the $\ell_{\infty,1}$ constraint.

{\bf Remark (lower bound on the alphabet size).} A simple lower bound is $\Omega(k^2)$. To see why, consider a graph with two nodes (i.e., $n=2$). Let $W$ be a $k$-by-$k$ weight matrix between the two nodes, defined as follows: $W(1,1)=W(2,2)=1$, $W(1,2)=W(2,1)=-1$, and $W(i,j)=0$ otherwise. This definition satisfies the condition that every row and column is centered (Definition~\ref{pairwise_model}). Besides, we have $\lambda = 1$ and $\eta = 1$, which means that the two quantities do not scale in $k$. To distinguish $W$ from the zero matrix, we need to observe samples in the set $\{(1,1), (2,2), (1,2), (2,1)\}$. This requires $\Omega(k^2)$ samples because any specific sample $(a,b)$ (where $a \in [k]$ and $b \in [k]$) has a probability of approximately $1/k^2$ to show up.

\subsection{Learning pairwise graphical models in $\tilde{O}(n^2)$ time}\label{sec:runtime}
Our results so far assume that the $\ell_1$-constrained logistic regression (in Algorithm~\ref{ising_learning_algo}) and the $\ell_{2,1}$-constrained logistic regression (in Algorithm~\ref{pairwise_learning_algo}) can be solved exactly. This would require $\tilde{O}(n^4)$ complexity if an interior-point based method is used~\citep{KKB07}. The goal of this section is to reduce the runtime to $\tilde{O}(n^2)$ via first-order optimization method. Note that $\tilde{O}(n^2)$ is an efficient time complexity for graph recovery over $n$ nodes. Previous structural learning algorithms of Ising models require either $\tilde{O}(n^2)$ complexity (e.g., \citep{Bre15, KM17}) or a worse complexity (e.g., \citep{RWL10, VMLC16} require $\tilde{O}(n^4)$ runtime).
%\footnote{If we change the $\ell_1$-regularized interaction screening estimator proposed in~\citep{VMLC16} to an $\ell_1$-constrained version, then it may be possible to apply the proposed mirror descent algorithm to optimize it. The main problem is that now the original proof of~\citep{VMLC16} does not work, and one needs a new proof for the $\ell_1$-constrained interaction screening estimator.}). 
We would like to remark that our goal here is not to give the fastest first-order optimization algorithm (see our remark after Theorem~\ref{pairwise_runtime}). Instead, our contribution is to provably show that it is possible to run Algorithm~\ref{ising_learning_algo} and Algorithm~\ref{pairwise_learning_algo} in $\tilde{O}(n^2)$ time without affecting the original statistical guarantees. 

To better exploit the problem structure\footnote{Specifically, for the $\ell_1$-constrained logisitic regression defined in~(\ref{ERM}), since the input sample satisifies $\norm{x}_{\infty}=1$, the loss function is $O(1)$-Lipschitz w.r.t. $\norm{\cdot}_1$. Similarly, for the $\ell_{2,1}$-constrained logisitic regression defined in~(\ref{pairwise_ERM}), the loss function is $O(1)$-Lipschitz w.r.t. $\norm{\cdot}_{2,1}$ because the input sample satisfies $\norm{x}_{2, \infty}=1$.}, we use the mirror descent algorithm\footnote{Other approaches include the standard projected gradient descent and the coordinate descent. Their convergence rates depend on either the smoothness or the Lipschitz constant (w.r.t. $\norm{\cdot}_2$) of the objective function~\citep{Bub15}. This would lead to a total runtime of $\tilde{O}(n^3)$ for our problem setting. Another option would be the composite gradient descent method, the analysis of which relies on the restricted strong convexity of the objective function~\citep{ANW10}. For other variants of mirror descent algorithms, see the remark after Theorem~\ref{pairwise_runtime}.} with a properly chosen distance generating function (aka the mirror map). Following the standard mirror descent setup, we use negative entropy as the mirror map for $\ell_1$-constrained logistic regression and a scaled group norm for $\ell_{2,1}$-constrained logistic regression (see Section 5.3.3.2 and Section 5.3.3.3 in~\citep{BN13} for more details). The pseudocode is given in Appendix~\ref{sec:mirror}. The main advantage of mirror descent algorithm is that its convergence rate scales \emph{logarithmically} in the dimension (see Lemma~\ref{convergence} in Appendix~\ref{sec:runtime_proof}). Specifically, let $\bar{w}$ be the output after $O(\ln(n)/\gamma^2)$ mirror descent iterations, then $\bar{w}$ satisfies
\begin{equation}
\hat{\Loss}(\bar{w}) - \hat{\Loss}(\hat{w}) \le \gamma, \label{algo_conv}
\end{equation}
where $\hat{\Loss}(w)=\sum_{i=1}^N \ln(1+e^{-y^i\inprod{w, x^i}}) /N$ is the empirical logistic loss, and $\hat{w}$ is the actual minimizer of $\hat{\Loss}(w)$. Since each mirror descent update requires $O(nN)$ time, where $N$ is the number of samples and scales as $O(\ln(n))$, and we have to solve $n$ regression problems (one for each variable in $[n]$), the total runtime scales as $\tilde{O}(n^2)$, which is our desired runtime.

There is still one problem left, that is, we have to show that $\norm{\bar{w}-w^*}_{\infty}\le \eps$ (where $w^*$ is the minimizer of the true loss $\Loss(w)=\E_{(x,y)\sim\D} \ln(1+e^{-y\inprod{w, x}})$) in order to conclude that Theorem~\ref{ising_theorem} and~\ref{pairwise_theorem} still hold when using mirror descent algorithms. Since $\hat{\Loss}(w)$ is not strongly convex, (\ref{algo_conv}) alone does not necessarily imply that $\norm{\bar{w}-\hat{w}}_{\infty}$ is small. Our key insight is that in the proof of Theorem~\ref{ising_theorem} and~\ref{pairwise_theorem}, the definition of $\hat{w}$ (as a minimizer of $\hat{\Loss}(w)$) is only used to show that $\hat{\Loss}(\hat{w})\le \hat{\Loss}(w^*)$ (see inequality (b) of (\ref{hat_bound}) in Appendix~\ref{sec:proof_kl_lemma}). It is then possible to replace this step with (\ref{algo_conv}) in the original proof, and prove that Theorem~\ref{ising_theorem} and~\ref{pairwise_theorem} still hold as long as $\gamma$ is small enough (see (\ref{w_bar_diff}) in Appendix~\ref{sec:runtime_proof}).

Our key results in this section are Theorem~\ref{ising_runtime} and Theorem~\ref{pairwise_runtime}, which show that Algorithm~\ref{ising_learning_algo} and Algorithm~\ref{pairwise_learning_algo} can run in $\tilde{O}(n^2)$ time without affecting the original statistical guarantees.

\begin{theorem}
In the setup of Theorem~\ref{ising_theorem}, suppose that the $\ell_1$-constrained logistic regression in Algorithm~\ref{ising_learning_algo} is optimized by the mirror descent method (Algorithm~\ref{l1_mirror_descent}) given in Appendix~\ref{sec:mirror}. Given $\rho\in(0,1)$ and $\eps>0$, if the number of mirror descent iterations satisfies $T=O(\lambda^2\exp(12\lambda)\ln(n)/\eps^4)$, and the number of samples satisfies $N = O(\lambda^2\exp(12\lambda)\ln(n/\rho)/\eps^4)$, then (\ref{close_A}) still holds with probability at least $1-\rho$. The total time complexity of Algorithm~\ref{ising_learning_algo} is $O(TNn^2)$.
\label{ising_runtime}
\end{theorem}

\begin{theorem}
In the setup of Theorem~\ref{pairwise_theorem}, suppose that the $\ell_{2,1}$-constrained logistic regression in Algorithm~\ref{pairwise_learning_algo} is optimized by the mirror descent method (Algorithm~\ref{l21_mirror_descent}) given in Appendix~\ref{sec:mirror}. Given $\rho\in(0,1)$ and $\eps>0$, if the number of mirror descent iterations satisfies $T=O(\lambda^2k^3\exp(12\lambda)\ln(n)/\eps^4)$, and the number of samples satisfies $N = O(\lambda^2k^4\exp(14\lambda)\ln(nk/\rho)/\eps^4)$, then (\ref{close_W}) still holds with probability at least $1-\rho$. The total time complexity of Algorithm~\ref{pairwise_learning_algo} is $O(TNn^2k^2)$.
\label{pairwise_runtime}
\end{theorem}

{\bf Remark.} It is possible to improve the time complexity given in Theorem~\ref{ising_runtime} and~\ref{pairwise_runtime} (especially the dependence on $\eps$ and $\lambda$), by using stochastic or accelerated versions of mirror descent algorithms (instead of the batch version given in Appendix~\ref{sec:mirror}). %For example, if online mirror descent algorithms are used, then the runtime would be $O(Nn^2)$ and $O(Nn^2k^2)$ simply because each mirror descent update uses a single sample instead of all samples (and the number of updates equals the number of samples). 
In fact, the Sparsitron algorithm proposed by \citet{KM17} can be seen as an online mirror descent algorithm for optimizing the $\ell_1$-constrained logistic regression (see Algorithm~\ref{l1_mirror_descent} in Appendix~\ref{sec:mirror}). Furthermore, Algorithm~\ref{ising_learning_algo} and~\ref{pairwise_learning_algo} can be parallelized as every node has an independent regression problem. 
%As pointed out at the beginning of this section, our goal here is not to give the most efficient optimization algorithm. The focus of this section is to show that it is possible to run Algorithm~\ref{ising_learning_algo} and Algorithm~\ref{pairwise_learning_algo} in $\tilde{O}(n^2)$ time and achieve the same statistical guarantee.

\section{Analysis}

\subsection{Proof outline}
We give a proof outline for Theorem~\ref{ising_theorem}. The proof of Theorem~\ref{pairwise_theorem} follows a similar outline. 
Let $D$ be a distribution over $\{-1,1\}^n\times \{-1,1\}$, where $(x,y)\sim D$ satisfies $\Pr[y= 1| x] = \sigma(\inprod{w^*,x})$. Let $\Loss(w) = \E_{(x,y)\sim\D} \ln(1+e^{-y\inprod{w,x}})$ and $\hat{\Loss}(w) = \sum_{i=1}^N \ln(1+e^{-y^i\inprod{w,x^i}})/N$ be the expected and empirical logistic loss. Suppose $\norm{w^*}_1\le 2\lambda$. Let $\hat{w}\in \argmin_w \hat{\Loss}(w)$ s.t. $\norm{w}_1\le 2\lambda$. Our goal is to prove that $\norm{\hat{w}-w^*}_{\infty}$ is small when the samples are constructed from an Ising model distribution. Our proof can be summarized in three steps:

\begin{enumerate}
\item If the number of samples satisfies $N= O(\lambda^2\ln(n/\rho)/\gamma^2)$, then $\Loss(\hat{w}) - \Loss(w^*) \le O(\gamma)$. This is obtained using a sharp generalization bound when $\norm{w}_{1}\le 2\lambda$ and $\norm{x}_{\infty}\le 1$ (see Lemma~\ref{gen_bound} in Appendix~\ref{sec:proof_kl_lemma}).
\item For any $w$, we show that $\Loss(w) - \Loss(w^*) \ge \E_x [\sigma(\inprod{w,x})- \sigma(\inprod{w^*,x})]^2$ (see Lemma~\ref{logistic_kl} and Lemma~\ref{pinsker} in Appendix~\ref{sec:proof_kl_lemma}). Hence, Step 1 implies that $\E_x [\sigma(\inprod{\hat{w},x})- \sigma(\inprod{w^*,x})]^2\le O(\gamma)$ (see Lemma~\ref{KL_lemma} in the next subsection).
\item We now use a result from~\citep{KM17} (see Lemma~\ref{infty_norm} in the next subsection), which says that if the samples are from an Ising model and if $\gamma = O(\eps^2 \exp(-6\lambda))$, then $\E_x [\sigma(\inprod{\hat{w},x})- \sigma(\inprod{w^*,x})]^2\le O(\gamma)$ implies that $\norm{\hat{w}-w^*}_{\infty}\le \eps$. The required number of samples is $N = O(\lambda^2\ln(n/\rho)/\gamma^2) = O(\lambda^2\exp(12\lambda)\ln(n/\rho)/\eps^4)$. 
\end{enumerate}

For the general setting with non-binary alphabet (i.e., Theorem~\ref{pairwise_theorem}), the proof is similar to that of Theorem~\ref{ising_theorem}. The main difference is that we need to use a sharp generalization bound when $\norm{w}_{2,1}\le 2\lambda \sqrt{k}$ and $\norm{x}_{2, \infty}\le 1$ (see Lemma~\ref{l21_gen_bound} in Appendix~\ref{sec:proof_kl_lemma}). This would give us Lemma~\ref{pairwise_KL_lemma} (instead of Lemma~\ref{KL_lemma} for the Ising models). The last step is to use Lemma~\ref{pairwise_infty_norm} to bound the infinity norm between the two weight matrices.

\subsection{Supporting lemmas}\label{sec_lemmas}

Lemma~\ref{KL_lemma} and Lemma~\ref{pairwise_KL_lemma} are the key results in our proof. They essentially say that given enough samples, solving the corresponding constrained logistic regression problem will provide a prediction $\sigma(\inprod{\hat{w},x})$ close to the true $\sigma(\inprod{w^*,x})$ in terms of their expected squared distance. 

\begin{lemma}
Let $\D$ be a distribution on $\{-1,1\}^n\times \{-1,1\}$ where for $(X, Y)\sim \D$, $\Pr[Y=1|X=x] = \sigma(\inprod{w^*,x})$. We assume that $\norm{w^*}_1\le 2\lambda$ for a known $\lambda \ge 0$. Given $N$ i.i.d. samples $\{(x^i, y^i)\}_{i=1}^N$, let $\hat{w}$ be any minimizer of the following $\ell_1$-constrained logistic regression problem:
\begin{equation}
\hat{w} \in \argmin_{w\in\R^n} \frac{1}{N} \sum_{i=1}^N \ln(1+e^{-y^i\inprod{w,x^i}}) \quad \text{\emph{ s.t. }} \norm{w}_1\le 2\lambda.
\end{equation}
Given $\rho\in(0,1)$ and $\eps>0$, if the number of samples satisfies $N=O(\lambda^2\ln(n/\rho)/\eps^2)$, then with probability at least $1-\rho$ over the samples, $\E_{(x,y)\sim\D}[(\sigma(\inprod{w^*,x}) - \sigma(\inprod{\hat{w},x}))^2] \le \eps$.
\label{KL_lemma}
\end{lemma}

\begin{lemma}
Let $\D$ be a distribution on $\mathcal{X}\times \{-1, 1\}$, where $\mathcal{X}=\{x\in\{0,1\}^{n\times k}: \norm{x}_{2, \infty}\le 1\}$. Furthermore, $(X, Y)\sim \D$ satisfies $\Pr[Y=1|X=x] = \sigma(\inprod{w^*,x})$, where $w^*\in \R^{n\times k}$. We assume that $\norm{w^*}_{2,1} \le 2\lambda \sqrt{k}$ for a known $\lambda \ge 0$. Given $N$ i.i.d. samples $\{(x^i, y^i)\}_{i=1}^N$ from $\D$, let $\hat{w}$ be any minimizer of the following $\ell_{2,1}$-constrained logistic regression problem:
\begin{equation}
\hat{w} \in \argmin_{w\in\R^{n\times k}} \frac{1}{N} \sum_{i=1}^N \ln(1+e^{-y^i\inprod{w, x^i}}) \quad \text{\emph{ s.t. }} \norm{w}_{2,1}\le 2\lambda\sqrt{k}.
\end{equation}
Given $\rho\in(0,1)$ and $\eps>0$, if the number of samples satisfies $N=O(\lambda^2k(\ln(n/\rho))/\eps^2)$, then with probability at least $1-\rho$ over the samples, $\E_{(x,y)\sim\D}[(\sigma(\inprod{w^*,x}) - \sigma(\inprod{\hat{w},x}))^2] \le \eps$.
\label{pairwise_KL_lemma}
\end{lemma}

The proofs of Lemma~\ref{KL_lemma} and Lemma~\ref{pairwise_KL_lemma} are given in Appendix~\ref{sec:proof_kl_lemma}. Note that in the setup of both lemmas, we form a pair of dual norms for $x$ and $w$, e.g., $\norm{x}_{2,\infty}$ and $\norm{w}_{2,1}$ in Lemma~\ref{pairwise_KL_lemma}, and $\norm{x}_{\infty}$ and $\norm{w}_1$ in Lemma~\ref{KL_lemma}. This duality allows us to use a sharp generalization bound with a sample complexity that scales logarithmic in the dimension (see Lemma~\ref{gen_bound} and Lemma~\ref{l21_gen_bound} in Appendix~\ref{sec:proof_kl_lemma}).  

%Intuitively, if a graphical model distribution concentrates over a subset of configurations in $[k]^n$, then it is difficult to learn the underlying graph. 
%One key property of the graphical model distribution is that this bad event cannot happen. 
%One property of the graphical model distribution is that the (conditional) probability of a variable taking any value in the alphabet is lowered bounded by a nonzero quantity (see Definition~\ref{unbiased} and Lemma~\ref{delta_unbias}). 
Definition~\ref{unbiased} defines a $\delta$-unbiased distribution. This notion of $\delta$-unbiasedness is proposed by \citet{KM17}.

\begin{definition} 
Let $S$ be the alphabet set, e.g., $S=\{-1,1\}$ for Ising model and $S=[k]$ for an alphabet of size $k$. A distribution $\D$ on $S^n$ is $\delta$-unbiased if for $X\sim\D$, any $i\in[n]$, and any assignment $x\in S^{n-1}$ to $X_{-i}$, $\min_{\alpha\in S}(\Pr[X_i= \alpha | X_{-i}=x])\ge \delta$.
\label{unbiased}
\end{definition}

For a $\delta$-unbiased distribution, any of its marginal distribution is also $\delta$-unbiased (see Lemma~\ref{marginal_unbiased}).
\begin{lemma}
Let $\D$ be a $\delta$-unbiased distribution on $S^n$, where $S$ is the alphabet set. For $X\sim\D$, any $i\in[n]$, the distribution of $X_{-i}$ is also $\delta$-unbiased.
\label{marginal_unbiased}
\end{lemma}

Lemma~\ref{delta_unbias} describes the $\delta$-unbiased property of graphical models. This property has been used in the previous papers (e.g., \citep{KM17, Bre15}). 
%For completeness, we give a proof of Lemma~\ref{marginal_unbiased} and Lemma~\ref{delta_unbias} in the appendix.
\begin{lemma}
Let $\D(\W, \Theta)$ be a pairwise graphical model distribution with alphabet size $k$ and width $\lambda(\W, \Theta)$. Then $\D(\W, \Theta)$ is $\delta$-unbiased with $\delta=e^{-2\lambda(\W, \Theta)}/k$. Specifically, an Ising model distribution $\D(A, \theta)$ is $e^{-2\lambda(A, \theta)}/2$-unbiased.
\label{delta_unbias}
\end{lemma}

In Lemma~\ref{KL_lemma} and Lemma~\ref{pairwise_KL_lemma}, we give a sample complexity bound for achieving a small $\ell_2$ error between $\sigma(\inprod{\hat{w},x})$ and $\sigma(\inprod{w^*,x})$. The following two lemmas show that if the sample distribution is $\delta$-unbiased, then a small $\ell_2$ error implies a small distance between $\hat{w}$ and $w^*$. 

\begin{lemma}
Let $\D$ be a $\delta$-unbiased distribution on $\{-1,1\}^n$. Suppose that for two vectors $u,w\in\R^n$ and $\theta', \theta''\in\R$,
$\E_{X\sim\D}[(\sigma(\inprod{w,X}+\theta') - \sigma(\inprod{u,X}+\theta''))^2] \le \eps$, where $\eps<\delta e^{-2\norm{w}_1-2|\theta'|-6}$. Then $\norm{w-u}_{\infty} \le O(1)\cdot e^{\norm{w}_1+|\theta'|}\cdot \sqrt{\eps/\delta}$.
\label{infty_norm}
\end{lemma}

\begin{lemma}
Let $\D$ be a $\delta$-unbiased distribution on $[k]^n$. For $X\sim\D$, let $\tilde{X}\in\{0,1\}^{n\times k}$ be the one-hot encoded $X$. Let $u, w\in\R^{n\times k}$ be two matrices satisfying $\sum_{a} u(i, a)=0$ and $\sum_{a} w(i, a)=0$, for $i\in[n]$. Suppose that for some $u, w$ and $\theta', \theta''\in\R$, we have $\E_{X\sim\D}[(\sigma(\inprod{w,\tilde{X}}+\theta') - \sigma(\inprod{u,\tilde{X}}+\theta''))^2] \le \eps$, where $\eps<\delta e^{-2\norm{w}_{\infty,1}-2|\theta'|-6}$. Then\footnote{For a matrix $w$, we define $\norm{w}_{\infty}=\max_{ij}|w(i,j)|$. Note that this is different from the induced matrix norm.} $\norm{w-u}_{\infty} \le O(1)\cdot e^{\norm{w}_{\infty,1}+|\theta'|}\cdot \sqrt{\eps/\delta}$.
\label{pairwise_infty_norm}
\end{lemma}

The proofs of Lemma~\ref{infty_norm} and Lemma~\ref{pairwise_infty_norm} can be found in~\citep{KM17} (see Claim 8.6 and Lemma 4.3 in their paper). We give a slightly different proof of these two lemmas in Appendix~\ref{sec:proof_infty_norm}.

\subsection{Proof sketches}
We provide proof sketches for Theorem~\ref{ising_theorem} and Theorem~\ref{pairwise_theorem} using the supporting lemmas. The detailed proof can be found in Appendix~\ref{sec:proof_ising} and~\ref{sec:proof_pairwise}.

{\bf Proof sketch of Theorem~\ref{ising_theorem}.} Without loss of generality, let us consider the $n$-th variable. Let $Z\sim\D(A, \theta)$, and $X=[Z_{-n}, 1] = [Z_1, Z_2, \cdots, Z_{n-1},1]\in\{-1,1\}^n$. By Fact \ref{logistic_dist} and Lemma~\ref{KL_lemma}, if $N=O(\lambda^2\ln(n/\rho)/\gamma^2)$, then $\E_{X}[(\sigma(\inprod{w^*,X}) - \sigma(\inprod{\hat{w},X}))^2] \le \gamma$ with probability at least $1-\rho/n$. By Lemma~\ref{delta_unbias} and Lemma~\ref{marginal_unbiased}, $Z_{-n}$ is $\delta$-unbiased with $\delta = e^{-2\lambda}/2$. We can then apply Lemma~\ref{infty_norm} to show that if $N=O(\lambda^2\exp(12\lambda)\ln(n/\rho)/\eps^4)$, then $\max_{j\in[n]} |A_{nj} - \hat{A}_{nj}| \le \eps$ with probability at least $1-\rho/n$. Theorem~\ref{ising_theorem} then follows by a union bound over all $n$ variables.

{\bf Proof sketch of Theorem~\ref{pairwise_theorem}.} Let us again consider the $n$-th variable since the proof is the same for all other variables. As described before, the key step is to show that (\ref{u_diff_w}) holds. Now fix a pair of $\alpha \neq \beta \in [k]$, let $N^{\alpha, \beta}$ be the number of samples such that the $n$-th variable is either $\alpha$ or $\beta$. By Fact~\ref{pairwise_logistic_dist} and Lemma~\ref{pairwise_KL_lemma},  if $N^{\alpha, \beta} = O(\lambda^2 k \ln(n/\rho')/ \gamma^2)$, then with probability at least $1-\rho'$, the matrix $U^{\alpha, \beta} \in \R^{n\times k}$ satisfies $\E_{x}[(\sigma(\inprod{w^*,x}) - \sigma(\inprod{U^{\alpha, \beta},x}))^2] \le \gamma$, where $w^*\in\R^{n\times k}$ is defined in (\ref{w_mat_def}). By Lemma~\ref{pairwise_infty_norm} and Lemma~\ref{delta_unbias}, if $N^{\alpha, \beta} = O(\lambda^2k^3 \exp(12\lambda) \ln(n/\rho'))/\eps^4)$, then with probability at least $1-\rho'$, $|W_{nj}(\alpha, b) - W_{nj}(\beta, b)-U^{\alpha, \beta}(j, b)| \le \eps, \;\forall j\in[n-1],\;\forall b\in[k]$. Since $\D(\W, \Theta)$ is $\delta$-unbiased with $\delta=e^{-2\lambda}/k$, in order to have $N^{\alpha, \beta}$ samples for a given $(\alpha, \beta)$ pair, we need the total number of samples to satisfy $N =O(N^{\alpha, \beta}/\delta)$. Theorem~\ref{pairwise_theorem} then follows by setting $\rho' = \rho/(nk^2)$ and taking a union bound over all $(\alpha, \beta)$ pairs and all $n$ variables.

\section{Experiments}
In both simulations below, we assume that the external field is zero. Sampling is done via exactly computing the distribution.

{\bf Learning Ising models.} In Figure~\ref{incoherence_example} we construct a diamond-shape graph and show that the incoherence value at Node 1 becomes bigger than 1 (and hence violates the incoherence condition in~\citep{RWL10}) when we increase the graph size $n$ and edge weight $a$. We then run 100 times of Algorithm~\ref{ising_learning_algo} and plot the fraction of runs that exactly recovers the underlying graph structure. In each run we generate a different set of samples. The result shown in Figure~\ref{incoherence_example} is consistent with our analysis and also indicates that our conditions for graph recovery are weaker than those in~\citep{RWL10}.
\begin{figure}[ht]

\centering
\begin{subfigure}	
    \centering
	\includegraphics[scale=0.5]{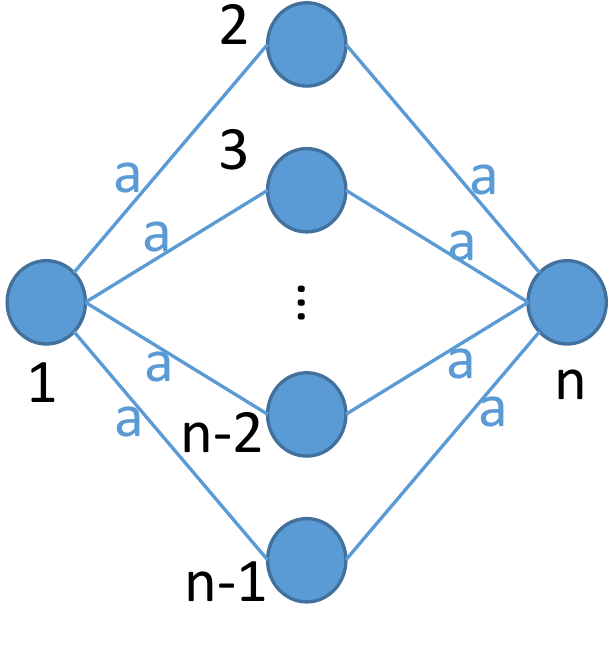}
\end{subfigure}
\begin{subfigure}
    \centering
	\includegraphics[width=0.45\textwidth]{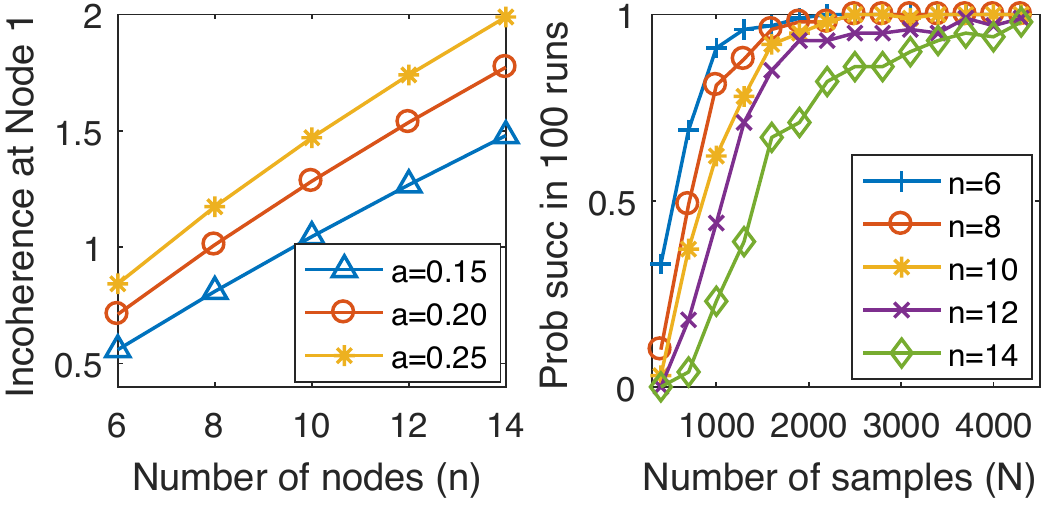}
\end{subfigure}

\caption{{\bf Left}: The graph structure used in this simulation. It has $n$ nodes and $2(n-2)$ edges. Every edge has the same weight $a>0$. {\bf Middle}: Incoherence value at Node 1. {\bf Right}: We simulate 100 runs of Algorithm~\ref{ising_learning_algo} for the diamond graph with edge weight $a=0.2$.} 
\label{incoherence_example}

\end{figure}

{\bf Learning general pairwise graphical models.} We compare our algorithm (Algorithm~\ref{pairwise_learning_algo}) with the Sparsitron algorithm in~\citep{KM17} on a two-dimensional 3-by-3 grid (shown in Figure~\ref{grid_graph}). We experiment two alphabet sizes: $k=4,6$. For each value of $k$, we simulate both algorithms 100 runs, and in each run we generate random $W_{ij}$ matrices with entries $\pm 0.2$. As shown in the Figure~\ref{non_binary_comp}, our algorithm requires fewer samples for successfully recovering the graphs. More details about this experiment can be found in Appendix~\ref{sec:more_exp}.

\begin{figure}[ht]
\centering
\begin{subfigure}	
    \centering
	\includegraphics[scale=0.5]{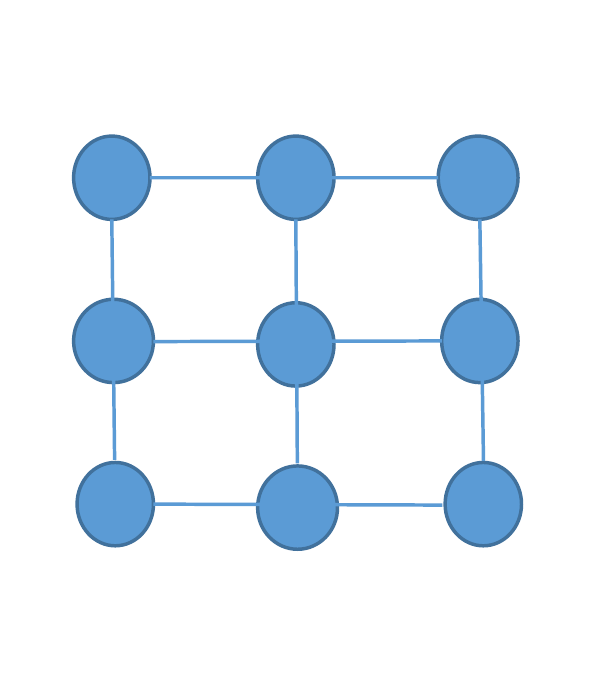}
\end{subfigure}
\begin{subfigure}
    \centering
	\includegraphics[width=0.45\textwidth]{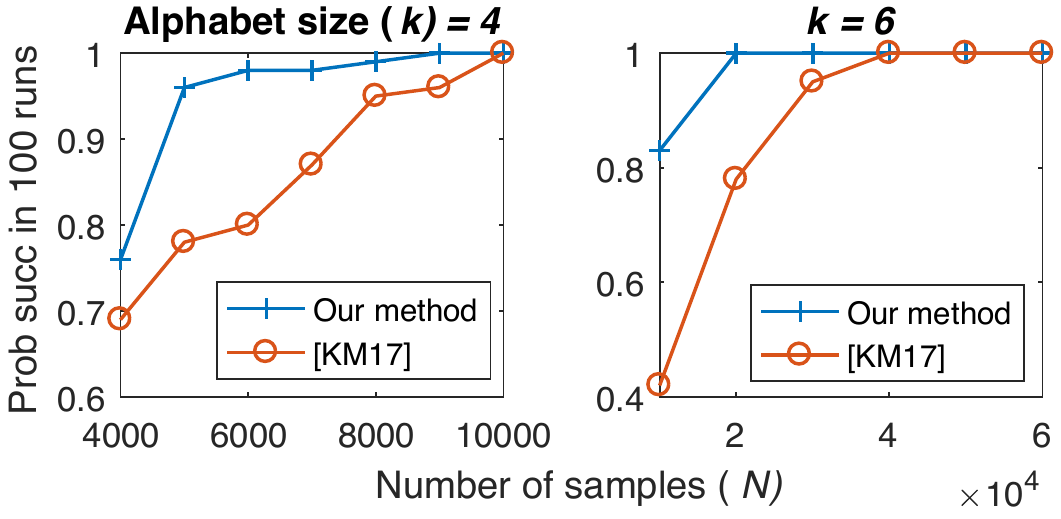}
\end{subfigure}
\caption{{\bf Left}: A two-dimensional 3-by-3 grid graph used in the simulation. {\bf Middle and right}: Our algorithm needs fewer samples than the Sparsitron algorithm for graph recovery.} 
\label{grid_graph}
\label{non_binary_comp}
\end{figure}
\section{Conclusion}
We have shown that the $\ell_{2,1}$-constrained logistic regression can recover the Markov graph of any discrete pairwise graphical model from i.i.d. samples. For the Ising model, it reduces to the $\ell_1$-constrained logistic regression. This algorithm has better sample complexity than the previous state-of-the-art result ($k^4$ versus $k^5$), and can run in $\tilde{O}(n^2)$ time. One interesting direction for future work is to see if the $1/\eta^4$ dependency in the sample complexity can be improved. 
Another interesting direction is to consider MRFs with higher-order interactions. %Higher-order MRFs can be reformulated to pairwise MRFs by adding auxiliary variables~\citep{WJ08}. 

\newpage
\bibliography{ref.bib}
\bibliographystyle{apalike}

\appendix

\section{Related work on learning Ising models}\label{sec:comps}

For the special case of learning Ising models (i.e., binary variables), we compare the sample complexity among different graph recovery algorithms in Table~\ref{table:comps}. 
{\renewcommand{\arraystretch}{1}%
\begin{table} [htbp]
\begin{center}
 \begin{tabular}{|p{3.3cm}|p{7.1cm}|p{3.8cm}|} 
 \hline
Paper  & Assumptions & Sample complexity ($N$) \\
 \hline
	\multirow{5}{3.3cm}{Information-theoretic lower bound \citep{SW12}} & 1. Model width $\le \lambda$, and $\lambda\ge 1$ &$\max\{\frac{\ln(n)}{2\eta \tanh(\eta)}$,\\
	& 2. Degree $\le d$ & $\frac{d}{8}\ln(\frac{n}{8d})$,\\
	& 3. Minimum edge weight $\ge \eta>0$ & \multirow{2}{3.8cm}{$\frac{\exp(\lambda)\ln(nd/4-1)}{4\eta d \exp(\eta)}\}$}\\ 
	& 4. External field $=0$ &\\
	&&\\
 \hline
	\multirow{11}{3.3cm}{$\ell_1$-regularized logistic regression~\citep{RWL10}} &$Q^*$ is the Fisher information matrix, & \multirow{11}{3.8cm}{$O(d^3\ln(n))$} \\ 
	& $S$ is set of neighbors of a given variable. &\\
	& 1. Dependency: $\exists$ $C_{\min}>0$ such that &\\
	& \hspace{2.1cm} eigenvalues of $Q^*_{SS}\ge C_{\min}$ & \\
	& 2. Incoherence: $\exists$ $\alpha \in(0,1]$ such that &\\
	& \hspace{2.1cm}$\norm{Q^*_{S^cS}(Q^*_{SS})^{-1}}_{\infty}\le 1-\alpha$ &\\
	& 3. Regularization parameter: &\\
	& \hspace{2.1cm} $\lambda_N\ge \frac{16(2-\alpha)}{\alpha}\sqrt{\frac{\ln(n)}{N}}$ &\\
	& 4. Minimum edge weight $\ge 10\sqrt{d}\lambda_N/C_{\min}$ &\\
	& 5. External field $=0$ &\\
	& 6. Probability of success $\ge 1-2e^{-O(\lambda_N^2N)}$ &\\
 \hline
	\multirow{4}{3.3cm}{Greedy algorithm~\citep{Bre15}} & 1. Model width $\le \lambda$ & \multirow{4}{3.8cm}{$O(\exp(\frac{\exp(O(d\lambda))}{\eta^{O(1)}})\ln(\frac{n}{\rho}))$}\\
	& 2. Degree $\le d$ &\\
	& 3. Minimum edge weight $\ge \eta>0$ &\\
	& 4. Probability of success $\ge 1-\rho$ &\\
\hline
\multirow{5}{3.3cm}{Interaction Screening~\citep{VMLC16}} & 1. Model width $\le \lambda$ &\\
	& 2. Degree $\le d$ &$O(\max\{d, \frac{1}{\eta^2}\}$\\
	& 3. Minimum edge weight $\ge \eta>0$ &$d^3\exp(6\lambda)\ln(\frac{n}{\rho}))$\\
	& 4. Regularization parameter $=4\sqrt{\frac{\ln(3n^2/\rho)}{N}}$ &\\
	& 5. Probability of success $\ge 1-\rho$ &\\
\hline
\multirow{5}{3.3cm}{$\ell_1$-regularized logistic regression~\citep{LVMC18}} & 1. Model width $\le \lambda$ &\\
    & 2. Degree $\le d$ &$O(\max\{d, \frac{1}{\eta^2}\}$ \\
	& 3. Minimum edge weight $\ge \eta>0$ &$d^3\exp(8\lambda)\ln(\frac{n}{\rho}))$\\
	& 4. Regularization parameter $O(\sqrt{\frac{\ln(n^2/\rho)}{N}})$ &\\
	& 5. Probability of success $\ge 1-\rho$ &\\
\hline
\multirow{4}{3.3cm}{$\ell_1$-constrained logistic regression~\citep{RH17}} & 1. Model width $\le \lambda$  & \multirow{3}{3.8cm}{$O(\frac{\lambda^2\exp(8\lambda)}{\eta^4}\ln(\frac{n}{\rho}))$} \\
	&2. Minimum edge weight $\ge\eta>0$ &\\
	&3. Probability of success $\ge 1-\rho$ &\\
	&&\\
\hline
	\multirow{3}{3.3cm}{Sparsitron~\citep{KM17}} & 1. Model width $\le \lambda$  & \multirow{3}{3.8cm}{$O(\frac{\lambda^2\exp(12\lambda)}{\eta^4}\ln(\frac{n}{\rho\eta}))$} \\
	&2. Minimum edge weight $\ge\eta>0$ &\\
	&3. Probability of success $\ge 1-\rho$ &\\
\hline
	\multirow{3}{3.3cm}{$\ell_1$-constrained logistic regression [{\bf this paper}]} & 1. Model width $\le \lambda$ & \multirow{3}{3.8cm}{$O(\frac{\lambda^2\exp(12\lambda)}{\eta^4}\ln(\frac{n}{\rho}))$}\\
	&2. Minimum edge weight $\ge\eta>0$ &\\ 
	&3. Probability of success $\ge 1-\rho$ &\\
\hline 
\end{tabular}
\end{center}
\vspace{-0.35cm}
\caption{Sample complexity comparison for learning Ising models. The second column lists the assumptions in their analysis. Given $\lambda$ and $\eta$, the degree is bounded by $d\le \lambda/\eta$, with equality achieved when every edge has the same weight and there is no external field.}
\label{table:comps}
\end{table}}

Note that the algorithms in~\citep{RWL10, Bre15, VMLC16, LVMC18} are designed for learning Ising models instead of general pairwise graphical models. Hence, they are not presented in Table~\ref{table:comps_non_binary}. 
%Our results show that $\ell_1$-constrained logistic regression can recover the underlying graph from i.i.d. samples of an Ising model. We make no incoherence assumptions and achieve the state-of-the-art sample complexity. 

As mentioned in the Introduction, \citet{RWL10} consider $\ell_1$-regularized logistic regression for learning Ising models in the high-dimensional setting. They require incoherence assumptions that ensure, via conditions on sub-matrices of the Fisher information matrix, that sparse predictors of each node are hard to confuse with a false set. Their analysis obtains significantly better sample complexity compared to what is possible when these extra assumptions are not imposed (see \citep{BM09}). Others have also considered $\ell_1$-regularization (e.g., \citep{LGK07,YL07, BGA08, JRVS11, YALR12, AE12}) for structure learning of Markov random fields but they all require certain assumptions about the graphical model and hence their methods do not work for general graphical models. The analysis of~\citep{RWL10} is of essentially the same convex program as this work (except that we have an additional thresholding procedure). The main difference is that they obtain a better sample guarantee but require significantly more restrictive assumptions. 

In the general setting with no restrictions on the model, \citet{SW12} provide an information-theoretic lower bound on the number of samples needed for graph recovery. This lower bound depends logarithmically on $n$, and exponentially on the width $\lambda$, and (somewhat inversely) on the minimum edge weight $\eta$. We will find these general broad trends, but with important differences, in the other algorithms as well.

\citet{Bre15} provides a greedy algorithm and shows that it can learn with sample complexity that grows logarithmically in $n$, but {\em doubly} exponentially in the width $\lambda$ and also exponentially in $1/\eta$. It is thus suboptimal with respect to its dependence on $\lambda$ and $\eta$. 

\citet{VMLC16} propose a new convex program (i.e. different from logistic regression), and for this they are able to show a single-exponential dependence on $\lambda$. There is also low-order polynomial dependence on $\lambda$ and $1/\eta$.
%the exact form of which depends on model setting since they present their results in terms of the degree $d$ of the graph as well as $\lambda$ and $\eta$. 
Note that given $\lambda$ and $\eta$, the degree is bounded by $d\le \lambda/\eta$ (the equality is achieved when every edge has the same weight and there is no external field). Therefore, their sample complexity can scale as worse as $1/\eta^5$. Later, the same authors~\citep{LVMC18} prove a similar result for the $\ell_1$-regularized logistic regression using essentially the same proof technique as~\citep{VMLC16}.

\citet{RH17} analyze the $\ell_1$-constrained logistic regression for learning Ising models. Their sample complexity\footnote{Lemma 5.21 in~\citep{RH17} has a typo: The upper bound should depend on $\exp(2\lambda)$. Accordingly, Theorem 5.23 should depend on $\exp(4\lambda)$ rather than $\exp(3\lambda)$.} has a better dependence on $1/\eta$ ($1/\eta^4$ vs $1/\eta^5$) than~\citep{LVMC18}. However, na\"ively extending their analysis to the $\ell_{2,1}$-constrained logistic regression will give a sample complexity exponential in the alphabet size\footnote{This is because the Hessian of the population loss has a lower bound that depends on $\exp(-2\lambda\sqrt{k})$ for $\norm{w}_{2,1} \le \lambda \sqrt{k}$ and $\norm{x}_{2, \infty} \le 1$.}.

In this paper, we analyze the $\ell_{2,1}$-constrained logistic regression for learning discrete pairwise graphical models with general alphabet. Our proof uses a sharp generalization bound for constrained logistic regression, which is different from~\citep{LVMC18, RH17}. For Ising models (shown in Table~\ref{table:comps}), our sample complexity matches that of~\citep{KM17}. For non-binary pairwise graphical models (shown in Table~\ref{table:comps_non_binary}), our sample complexity improves the state-of-the-art result.

\section{Proof of Lemma~\ref{KL_lemma} and Lemma~\ref{pairwise_KL_lemma}}\label{sec:proof_kl_lemma}
The proof of Lemma~\ref{KL_lemma} relies on the following lemmas. The first lemma is a generalization error bound for any Lipschitz loss of linear functions with bounded $\norm{w}_{1}$ and $\norm{x}_{\infty}$. 

\begin{lemma} (see, e.g., Corollary 4 of~\citep{KST09} and Theorem 26.15 of~\citep{SSBD14})
Let $\D$ be a distribution on $\mathcal{X}\times\mathcal{Y}$, where $\mathcal{X}=\{x\in\R^n: \norm{x}_{\infty}\le X_{\infty}\}$, and $\mathcal{Y}=\{-1,1\}$. Let $\ell:\R\to\R$ be a loss function with Lipschitz constant $L_{\ell}$. 
Define the expected loss $\Loss(w)$ and the empirical loss $\hat{\Loss}(w)$ as 
\begin{equation}
\Loss(w) = \E_{(x,y)\sim\D} \ell(y\inprod{w,x}), \quad \hat{\Loss}(w) = \frac{1}{N} \sum_{i=1}^N \ell(y^i\inprod{w,x^i}), 
\end{equation}
where $\{x^i,y^i\}_{i=1}^N$ are i.i.d. samples from distribution $\D$. Define $\mathcal{W}=\{w\in\R^n: \norm{w}_1\le W_1\}$.
Then with probability at least $1-\rho$ over the samples, we have that for all $w\in\mathcal{W}$,
\begin{equation}
\Loss(w) \le \hat{\Loss}(w) +  2L_{\ell}X_{\infty}W_1\sqrt{\frac{2\ln(2n)}{N}} + L_{\ell}X_{\infty}W_1\sqrt{\frac{2\ln(2/\rho)}{N}}.
\end{equation}
\label{gen_bound}
\end{lemma}
%Lemma~\ref{gen_bound} is essentially Theorem 26.15 of~\citep{SSBD14} (for the binary classification setup). 

\begin{lemma}(Pinsker's inequality)
Let $D_{KL}(a||b):= a\ln(a/b) + (1-a)\ln((1-a)/(1-b))$ denote the KL-divergence between two Bernoulli distributions $(a,1-a)$, $(b,1-b)$ with $a,b\in [0,1]$. Then
\begin{equation}
(a-b)^2 \le \frac{1}{2}D_{KL}(a||b).
\end{equation}
\label{pinsker}
\end{lemma}

%Lemma~\ref{pinsker} is simply the Pinsker's inequality applied to the binary distributions.

\begin{lemma}
Let $\D$ be a distribution on $\mathcal{X}\times \{-1, 1\}$. For $(X, Y)\sim \D$, $\Pr[Y=1|X=x] = \sigma(\inprod{w^*,x})$, where $\sigma(x) = 1/(1+e^{-x})$ is the sigmoid function. Let $\Loss(w)$ be the expected logistic loss: 
\begin{equation}
\Loss(w) = \E_{(x,y)\sim\D} \ln(1+e^{-y\inprod{w,x}}) = \E_{(x,y)\sim\D}[ -\frac{y+1}{2} \ln(\sigma(\inprod{w,x})) - \frac{1-y}{2} \ln(1-\sigma(\inprod{w,x}))].
\end{equation}
Then for any $w$, we have
\begin{equation}
\Loss(w)-\Loss(w^*) =\E_{(x,y)\sim\D}[D_{KL}(\sigma(\inprod{w^*,x})||\sigma(\inprod{w,x}))],
\end{equation}
where $D_{KL}(a||b):= a\ln(a/b) + (1-a)\ln((1-a)/(1-b))$ denotes the KL-divergence between two Bernoulli distributions $(a,1-a)$, $(b,1-b)$ with $a,b\in [0,1]$. 
\label{logistic_kl}
\end{lemma}
\begin{proof}
Simply plugging in the definition of the expected logistic loss $\Loss(\cdot)$ gives
\begin{align}
\Loss(w)- \Loss(w^*) & =  \E_{(x,y)\sim\D}[ -\frac{y+1}{2}  \ln(\sigma(\inprod{w,x})) - \frac{1-y}{2}  \ln(1-\sigma(\inprod{w,x}))] \nonumber \\
&\quad\quad +\E_{(x,y)\sim\D} [ \frac{y+1}{2} \ln(\sigma(\inprod{w^*,x})) + \frac{1-y}{2}  \ln(1-\sigma(\inprod{w^*,x}))] \nonumber\\
& = \E_x \E_{y|x} [ -\frac{y+1}{2} \ln(\sigma(\inprod{w,x})) - \frac{1-y}{2} \ln(1-\sigma(\inprod{w,x}))]  \nonumber \\
&\quad\quad +\E_x \E_{y|x}  [ \frac{y+1}{2} \ln(\sigma(\inprod{w^*,x})) + \frac{1-y}{2} \ln(1-\sigma(\inprod{w^*,x}))] \nonumber\\
& \overset{(a)}{=} \E_x  [-\sigma(\inprod{w^*,x})\ln(\sigma(\inprod{w,x})) - (1-\sigma(\inprod{w^*,x}))\ln(1-\sigma(\inprod{w,x}))] \nonumber\\
&\quad\quad + \E_x [\sigma(\inprod{w^*,x})\ln(\sigma(\inprod{w^*,x})) + (1-\sigma(\inprod{w^*,x}))\ln(1-\sigma(\inprod{w^*,x}))] \nonumber\\
& = \E_x \left[\sigma(\inprod{w^*,x})\ln\left(\frac{\sigma(\inprod{w^*,x}) }{\sigma(\inprod{w,x})}\right)+(1- \sigma(\inprod{w^*,x}))\ln\left(\frac{1-\sigma(\inprod{w^*,x})}{1-\sigma(\inprod{w,x})}\right)\right] \nonumber \\
& = \E_{(x,y)\sim\D}[D_{KL}(\sigma(\inprod{w^*,x})||\sigma(\inprod{w,x}))], \nonumber
\end{align}
where (a) follows from the fact that 
\begin{equation}
E_{y|x}[y] = 1\cdot \Pr[y=1|x]  + (- 1)\cdot \Pr[y=-1|x]  = 2\sigma(\inprod{w^*,x})-1. \nonumber
\end{equation}
\end{proof}

We are now ready to prove Lemma~\ref{KL_lemma} (which is restated below):
\begin{app-lemma}
Let $\D$ be a distribution on $\{-1,1\}^n\times \{-1,1\}$ where for $(X, Y)\sim \D$, $\Pr[Y=1|X=x] = \sigma(\inprod{w^*,x})$. We assume that $\norm{w^*}_1\le 2\lambda$ for a known $\lambda \ge 0$. Given $N$ i.i.d. samples $\{(x^i, y^i)\}_{i=1}^N$, let $\hat{w}$ be any minimizer of the following $\ell_1$-constrained logistic regression problem:
\begin{equation}
\hat{w} \in \arg \min_{w\in\R^n} \frac{1}{N} \sum_{i=1}^N \ln(1+e^{-y^i\inprod{w,x^i}}) \quad \text{\emph{ s.t. }} \norm{w}_1\le 2\lambda.
\end{equation}
Given $\rho\in(0,1)$ and $\eps>0$, suppose that $N=O(\lambda^2(\ln(n/\rho))/\eps^2)$, then with probability at least $1-\rho$ over the samples, we have that $\E_{(x,y)\sim\D}[(\sigma(\inprod{w^*,x}) - \sigma(\inprod{\hat{w},x}))^2] \le \eps$.
\end{app-lemma}

\begin{proof}
We first apply Lemma~\ref{gen_bound} to the setup of Lemma~\ref{KL_lemma}. The loss function $\ell(z) = \ln(1+e^{-z})$ defined above has Lipschitz constant $L_{\ell}=1$. The input sample $x\in\{-1,1\}^n$ satisfies $\norm{x}_{\infty}\le 1$. Let $\mathcal{W}=\{w\in\R^{n\times k}: \norm{w}_{1} \le 2\lambda\}$. According to Lemma~\ref{gen_bound}, with probability at least $1-\rho/2$ over the draw of the training set, we have that for all $w\in\mathcal{W}$, 
\begin{equation}
\Loss(w) \le \hat{\Loss}(w) +  4\lambda \sqrt{\frac{2\ln(2n)}{N}} + 2\lambda\sqrt{\frac{2\ln(4/\rho)}{N}}. \label{risk_bound}
\end{equation}
where $\Loss(w)=\E_{(x,y)\sim\D} \ln(1+e^{-y\inprod{w, x}}) $ and $\hat{\Loss}(w)=\sum_{i=1}^N \ln(1+e^{-y^i\inprod{w, x^i}}) /N$ are the expected loss and empirical loss.

Let $N= C\cdot\lambda^2\ln(8n/\rho)/\eps^2$ for a global constant $C$, then (\ref{risk_bound}) implies that with probability at least $1-\rho/2$, 
\begin{equation}
\Loss(w) \le \hat{\Loss}(w) + \eps, \text{ for all }w\in\mathcal{W}. \label{eps_bound}
\end{equation}

We next prove a concentration result for $\hat{\Loss}(w^*)$. Here $w^*$ is the true regression vector and is assumed to be fixed. First notice that $ \ln(1+e^{-y\inprod{w^*,x}})$ is bounded because $|y\inprod{w^*,x}|\le 2\lambda$. Besides, the $\ln(1+e^{-z})$ has Lipschitz 1, so $|\ln(1+e^{-2\lambda})-\ln(1+e^{2\lambda})|\le 4\lambda$. Hoeffding's inequality gives that $\Pr[\hat{\Loss}(w^*) -  \Loss(w^*) \ge t]  \le e^{-2Nt^2/(4\lambda)^2}$. Let $N= C'\cdot\lambda^2\ln(2/\rho)/\eps^2$ for a global constant $C'$, then with probability at least $1-\rho/2$ over the samples, 
\begin{equation}
\hat{\Loss}(w^*)  \le  \Loss(w^*) + \eps. \label{con_bound}
\end{equation}

Then the following holds with probability at least $1-\rho$:
\begin{equation}
\Loss(\hat{w}) \overset{(a)}{\le} \hat{\Loss}(\hat{w}) +  \eps \overset{(b)}{\le} \hat{\Loss}(w^*) +  \eps \overset{(c)}{\le} \Loss(w^*) + 2\eps, \label{hat_bound}
\end{equation} 
where (a) follows from (\ref{eps_bound}), (b) follows from the fact $\hat{w}$ is the minimizer of $\hat{\Loss}(w)$, and (c) follows from (\ref{con_bound}). 

So far we have shown that $\Loss(\hat{w}) - \Loss(w^*)\le 2\eps$ with probability at least $1-\rho$. The last step is to lower bound $\Loss(\hat{w}) - \Loss(w^*)$ by $\E_{(x,y)\sim\D} (\sigma(\inprod{w^*,x}) - \sigma(\inprod{w,x}))^2$ using Lemma~\ref{pinsker} and Lemma~\ref{logistic_kl}.
\begin{align}
\E_{(x,y)\sim\D} (\sigma(\inprod{w^*,x}) - \sigma(\inprod{w,x}))^2 &\overset{(d)}{\le} \E_{(x,y)\sim\D} D_{KL}(\sigma(\inprod{w^*,x})||\sigma(\inprod{w,x}))/2 \nonumber\\
& \overset{(e)}{=} (\Loss(\hat{w})- \Loss(w^*))/2 \nonumber\\
& \overset{(f)}{\le} \eps, \nonumber
\end{align}
where (d) follows from Lemma~\ref{pinsker}, (e) follows from Lemma~\ref{logistic_kl}, and (f) follows from (\ref{hat_bound}). Therefore, we have that $\E_{(x,y)\sim\D} (\sigma(\inprod{w^*,x}) - \sigma(\inprod{w,x}))^2 \le \eps$ with probability at least $1-\rho$, if the number of samples satisfies $N= O(\lambda^2\ln(n/\rho)/\eps^2)$.
\end{proof}

The proof of Lemma~\ref{pairwise_KL_lemma} is identical to the proof of Lemma~\ref{KL_lemma}, except that it relies on the following generalization error bound for Lipschitz loss functions with bounded $\ell_{2,1}$-norm. 
\begin{lemma}
Let $\D$ be a distribution on $\mathcal{X}\times\mathcal{Y}$, where $\mathcal{X}=\{x\in\R^{n\times k}: \norm{x}_{2, \infty}\le X_{2,\infty}\}$, and $\mathcal{Y}=\{-1,1\}$. Let $\ell:\R\to\R$ be a loss function with Lipschitz constant $L_{\ell}$. Define the expected loss $\Loss(w)$ and the empirical loss $\hat{\Loss}(w)$ as 
\begin{equation}
\Loss(w) = \E_{(x,y)\sim\D} \ell(y\inprod{w,x}), \quad \hat{\Loss}(w) = \frac{1}{N} \sum_{i=1}^N \ell(y^i\inprod{w,x^i}), 
\end{equation}
where $\{x^i,y^i\}_{i=1}^N$ are i.i.d. samples from distribution $\D$. Define $\mathcal{W}=\{w\in\R^{n\times k}: \norm{w}_{2,1}\le W_{2,1}\}$.
Then with probability at least $1-\rho$ over the draw of $N$ samples, we have that for all $w\in\mathcal{W}$,
\begin{equation}
\Loss(w) \le \hat{\Loss}(w)  + 2L_{\ell}X_{2,\infty}W_{2,1}\sqrt{\frac{6\ln(n)}{N}} + L_{\ell}X_{2,\infty}W_{2,1}\sqrt{\frac{2\ln(2/\rho)}{N}}.
\end{equation}
\label{l21_gen_bound}
\end{lemma}

Lemma~\ref{l21_gen_bound} can be readily derived from the existing results. First, notice that the dual norm of $\norm{\cdot}_{2,1}$ is $\norm{\cdot}_{2, \infty}$. Using Corollary 14 in~\citep{KST12}, Theorem 1 in~\citep{KST09}, and the fact that $\norm{w}_{2,q}\le \norm{w}_{2,1}$ for $q\ge 1$, we conclude that the Rademacher complexity of the function class $\mathcal{F}:=\{x\to \inprod{w,x}: \norm{w}_{2,1}\le W_{2,1}\}$ is at most $X_{2,\infty}W_{2,1}\sqrt{6\ln(n)/N}$. We can then obtain the standard Rademacher-based generalization bound (see, e.g.,~\citep{BM02} and Theorem 26.5 in~\citep{SSBD14}) for bounded Lipschitz loss functions. 

We omit the proof of Lemma~\ref{pairwise_KL_lemma} since it is the same as that of Lemma~\ref{KL_lemma}.

\section{Proof of Lemma~\ref{marginal_unbiased}}
Lemma~\ref{marginal_unbiased} is restated below.
\begin{app-lemma}
Let $\D$ be a $\delta$-unbiased distribution on $S^n$, where $S$ is the alphabet set. For $X\sim\D$, any $i\in[n]$, the distribution of $X_{-i}$ is also $\delta$-unbiased.
\end{app-lemma}

\begin{proof}
For any $j\neq i \in [n]$, any $a\in S$, and any $x\in S^{n-2}$, we have
\begin{align}
\Pr[X_j = a | X_{[n]\backslash\{i,j\}}=x] & = \sum_{b\in S} \Pr[X_j= a, X_i = b | X_{[n]\backslash\{i,j\}}=x] \nonumber\\
& = \sum_{b\in S} \Pr[X_i=b |  X_{[n]\backslash\{i,j\}}=x]\cdot \Pr[X_j=a | X_i=b, X_{[n]\backslash\{i,j\}}=x] \nonumber \\
& \overset{(a)}{\ge} \delta  \sum_{b\in S} \Pr[X_i=b |  X_{[n]\backslash\{i,j\}}=x] \nonumber\\
& = \delta, \label{margin_delta}
\end{align}
where (a) follows from the fact that $X\sim\D$ and $\D$ is a $\delta$-unbiased distribution. Since (\ref{margin_delta}) holds for any $j\neq i \in [n]$, any $a\in S$, and any $x\in S^{n-2}$, by definition, the distribution of $X_{-i}$ is $\delta$-unbiased.
\end{proof}

\section{Proof of Lemma~\ref{delta_unbias}}
The lemma is restated below, followed by its proof.
\begin{app-lemma}
Let $\D(\W, \Theta)$ be a pairwise graphical model distribution with alphabet size $k$ and width $\lambda(\W, \Theta)$. Then $\D(\W, \Theta)$ is $\delta$-unbiased with $\delta=e^{-2\lambda(\W, \Theta)}/k$. Specifically, an Ising model distribution $\D(A, \theta)$ is $e^{-2\lambda(A, \theta)}/2$-unbiased. 
\end{app-lemma}

\begin{proof}
Let $X\sim \D(\W, \Theta)$, and assume that $X\in [k]^n$. For any $i\in[n]$, any $a\in [k]$, and any $x\in [k]^{n-1}$, we have
\begin{align}
\Pr[X_i = a | X_{-i} =x] & = \frac{\exp(\sum_{j\neq i } W_{ij}(a, x_j) + \theta_{i}(a))}{\sum_{b\in [k]} \exp(\sum_{j\neq i } W_{ij}(b, x_j) + \theta_{i}(b))} \nonumber\\
& = \frac{1}{\sum_{b\in[k]}\exp(\sum_{j\neq i } (W_{ij}(b, x_j)-W_{ij}(a, x_j)) + \theta_{i}(b)-\theta_i(a))} \nonumber\\
& \overset{(a)}{\ge} \frac{1}{k\cdot \exp(2\lambda(\W, \Theta))} = e^{-2\lambda(\W, \Theta)}/k,
\end{align}
where (a) follows from the definition of model width. The lemma then follows (Ising model corresponds to the special case of $k=2$).
\end{proof}

\section{Proof of Lemma~\ref{infty_norm} and Lemma~\ref{pairwise_infty_norm}}\label{sec:proof_infty_norm}
The proof relies on the following basic property of the sigmoid function (see Claim 4.2 of~\citep{KM17}):
\begin{equation}
|\sigma(a)-\sigma(b)| \ge e^{-|a|-3} \cdot \min(1, |a-b|), \quad \forall a, b\in\R. 
\label{sigma_diff}
\end{equation}

We first prove Lemma~\ref{infty_norm} (which is restated below).
\begin{app-lemma}
Let $\D$ be a $\delta$-unbiased distribution on $\{-1,1\}^n$. Suppose that for two vectors $u,w\in\R^n$ and $\theta', \theta''\in\R$,
$\E_{X\sim\D}[(\sigma(\inprod{w,X}+\theta') - \sigma(\inprod{u,X}+\theta''))^2] \le \eps$, where $\eps<\delta e^{-2\norm{w}_1-2|\theta'|-6}$. Then $\norm{w-u}_{\infty} \le O(1)\cdot e^{\norm{w}_1+|\theta'|}\cdot \sqrt{\eps/\delta}$.
\end{app-lemma}

\begin{proof}
For any $i\in[n]$, any $X\in\{-1,1\}^n$, let $X_i\in \{-1,1\}$ be the $i$-th variable and $X_{-i}\in\{-1,1\}^{n-1}$ be the $[n]\backslash\{i\}$ variables. Let $X^{i, +}\in \{-1,1\}^n$ (respectively $X^{i, -}$) be the vector obtained from $X$ by setting $X_i=1$ (respectively $X_i=-1$). Then we have 
\begin{align}
\eps & \ge \E_{X\sim\D}[(\sigma(\inprod{w,X}+\theta') - \sigma(\inprod{u,X}+\theta''))^2] \nonumber \\
&= \E_{X_{-i}}\left[ \E_{X_{i}| X_{-i}}(\sigma(\inprod{w,X}+\theta') - \sigma(\inprod{u,X}+\theta''))^2\right] \nonumber\\
&= \E_{X_{-i}}[(\sigma(\inprod{w,X^{i,+}}+\theta') - \sigma(\inprod{u,X^{i,+}}+\theta''))^2\cdot \Pr[X_i=1|X_{-i}] \nonumber\\
& \quad \quad\quad +(\sigma(\inprod{w,X^{i,-}}+\theta') - \sigma(\inprod{u,X^{i,-}}+\theta''))^2\cdot \Pr[X_i=-1|X_{-i}] ]\nonumber \\
&\overset{(a)}{\ge} \delta\cdot \E_{X_{-i}} [(\sigma(\inprod{w,X^{i,+}}+\theta') - \sigma(\inprod{u,X^{i,+}}+\theta''))^2\nonumber\\
& \quad \quad\quad +(\sigma(\inprod{w,X^{i,-}}+\theta') - \sigma(\inprod{u,X^{i,-}}+\theta''))^2]\nonumber\\
&\overset{(b)}{\ge} \delta e^{-2\norm{w}_1-2|\theta'|-6}\cdot\E_{X_{-i}}[ \min(1, ((\inprod{w,X^{i,+}}+\theta') - (\inprod{u,X^{i,+}}+\theta''))^2)\nonumber \\
& \quad \quad\quad\quad \quad\quad\quad\quad\quad\quad + \min(1, ((\inprod{w,X^{i,-}}+\theta')-(\inprod{u,X^{i,-}}+\theta''))^2)]\nonumber\\
& \overset{(c)}{\ge}  \delta e^{-2\norm{w}_1-2|\theta'|-6}\cdot\E_{X_{-i}}  \min(1, (2w_i - 2u_i)^2/2) \nonumber\\
& \overset{(d)}{=} \delta e^{-2\norm{w}_1-2|\theta'|-6}\cdot\min(1, 2(w_i - u_i)^2).  \label{eps_lower}
\end{align} 
Here (a) follows from the fact that $\D$ is a $\delta$-unbiased distribution, which implies that $\Pr[X_i=1|X_{-i}] \ge \delta$ and $\Pr[X_i=-1|X_{-i}] \ge \delta$. Inequality (b) is obtained by substituting (\ref{sigma_diff}). Inequality (c) uses the following fact
\begin{equation}
\min(1, a^2)+\min(1, b^2) \ge \min (1, (a-b)^2/2), \forall a, b\in\R. \label{ab_ineq}
\end{equation}
To see why (\ref{ab_ineq}) holds, note that if both $|a|, |b|\le1$, then (\ref{ab_ineq}) is true since $a^2+b^2 \ge (a-b)^2/2$.  Otherwise, (\ref{ab_ineq}) is true because the left-hand side is at least 1 while the right-hand side is at most 1. The last equality (d) follows from that $X_{-i}$ is independent of $\min(1, 2(w_i - u_i)^2)$. 

Since $\eps<\delta e^{-2\norm{w}_1-2|\theta'|-6}$, (\ref{eps_lower}) implies that $|w_i - u_i| \le O(1)\cdot e^{\norm{w}_1+|\theta'|}\cdot \sqrt{\eps/\delta}$. Because (\ref{eps_lower}) holds for any $i\in[n]$, we have that $\norm{w-u}_{\infty} \le O(1)\cdot e^{\norm{w}_1+|\theta'|}\cdot \sqrt{\eps/\delta}$.
\end{proof}

We now prove Lemma~\ref{pairwise_infty_norm} (which is restated below).
\begin{app-lemma}
Let $\D$ be a $\delta$-unbiased distribution on $[k]^n$. For $X\sim\D$, let $\tilde{X}\in\{0,1\}^{n\times k}$ be the one-hot encoded $X$. Let $u, w\in\R^{n\times k}$ be two matrices satisfying $\sum_{j} u(i, j)=0$ and $\sum_{j} w(i, j)=0$ for $i\in[n]$. Suppose that for some $u, w$ and $\theta', \theta''\in\R$, we have $\E_{X\sim\D}[(\sigma(\inprod{w,\tilde{X}}+\theta') - \sigma(\inprod{u,\tilde{X}}+\theta''))^2] \le \eps$, where $\eps<\delta e^{-2\norm{w}_{\infty,1}-2|\theta'|-6}$. Then $\norm{w-u}_{\infty} \le O(1)\cdot e^{\norm{w}_{\infty,1}+|\theta'|}\cdot \sqrt{\eps/\delta}$.
\end{app-lemma}

\begin{proof}
Fix an $i\in[n]$ and $a\neq b\in[k]$. Let $X^{i, a}\in [k]^n$ (respectively $X^{i, b}$) be the vector obtained from $X$ by setting $X_i=a$ (respectively $X_i=b$). Let $\tilde{X}^{i,a}\in \{0,1\}^{n\times k}$ be the one-hot encoding of $X^{i, a}\in [k]^n$. Then we have
\begin{align}
\eps & \ge \E_{X\sim\D}[(\sigma(\inprod{w,\tilde{X}}+\theta') - \sigma(\inprod{u,\tilde{X}}+\theta''))^2] \nonumber \\
 & = \E_{X_{-i}}\left[ \E_{X_{i}| X_{-i}} (\sigma(\inprod{w,\tilde{X}}+\theta') - \sigma(\inprod{u,\tilde{X}}+\theta''))^2\right] \nonumber\\
 & \ge \E_{X_{-i}}[ (\sigma(\inprod{w,\tilde{X}^{i,a}}+\theta') - \sigma(\inprod{u,\tilde{X}^{i,a}}+\theta''))^2\cdot \Pr[X_i=a|X_{-i}]\nonumber\\
 &\quad\quad\quad + (\sigma(\inprod{w,\tilde{X}^{i,b}}+\theta') - \sigma(\inprod{u,\tilde{X}^{i,b}}+\theta''))^2\cdot \Pr[X_i=b|X_{-i}] ]\nonumber\\
 & \overset{(a)}{\ge} \delta e^{-2\norm{w}_{\infty,1}-2|\theta'|-6} \cdot \E_{X_{-i}}[ \min(1, ((\inprod{w,\tilde{X}^{i,a}}+\theta') - (\inprod{u,\tilde{X}^{i,a}}+\theta''))^2) \nonumber\\
 &\quad\quad\quad\quad\quad\quad\quad\quad\quad\quad\quad + \min(1, ((\inprod{w,\tilde{X}^{i,b}}+\theta') - (\inprod{u,\tilde{X}^{i,b}}+\theta''))^2)]\nonumber\\
 & \overset{(b)}{\ge}  \delta e^{-2\norm{w}_{\infty,1}-2|\theta'|-6} \cdot \E_{X_{-i}} \min(1, ((w(i,a)-w(i,b))-(u(i,a)-u(i,b)))^2/2) \nonumber\\
 & = \delta e^{-2\norm{w}_{\infty,1}-2|\theta'|-6} \min(1, ((w(i,a)-w(i,b))-(u(i,a)-u(i,b)))^2/2) \label{pairwise_eps_lower}
\end{align}
Here (a) follows from that $\D$ is a $\delta$-unbiased distribution and (\ref{sigma_diff}). Inequality (b) follows from (\ref{ab_ineq}). Because $\eps<\delta e^{-2\norm{w}_{\infty,1}-2|\theta'|-6}$, (\ref{pairwise_eps_lower}) implies that
\begin{align}
(w(i,a)-w(i,b))-(u(i,a)-u(i,b)) & \le O(1)\cdot e^{\norm{w}_{\infty,1}+|\theta'|}\cdot\sqrt{\eps/\delta}. \label{wu_diff}\\
(u(i,a)-u(i,b))-(w(i,a)-w(i,b)) & \le O(1)\cdot e^{\norm{w}_{\infty,1}+|\theta'|}\cdot\sqrt{\eps/\delta}. \label{uw_diff}
\end{align}
Since (\ref{wu_diff}) and (\ref{uw_diff}) hold for any $a\neq b\in [k]$, we can sum over $b\in[k]$ and use the fact that $\sum_{j} u(i, j)=0$ and $\sum_{j} w(i, j)=0$ to get
\begin{align}
w(i,a)- u(i,a) & = \frac{1}{k}\sum_b (w(i,a)-w(i,b))-(u(i,a)-u(i,b)) \le O(1)\cdot e^{\norm{w}_{\infty,1}+|\theta'|}\cdot\sqrt{\eps/\delta}.\nonumber\\
u(i,a)- w(i,a) & = \frac{1}{k}\sum_b (u(i,a)-u(i,b))-(w(i,a)-w(i,b)) \le O(1)\cdot e^{\norm{w}_{\infty,1}+|\theta'|}\cdot\sqrt{\eps/\delta}.\nonumber
\end{align}
Therefore, we have $|w(i,a) - u(i, a)|\le O(1)\cdot e^{\norm{w}_{\infty,1}+|\theta'|}\cdot\sqrt{\eps/\delta}$, for any $i\in[n]$ and $a\in [k]$.
\end{proof}

\section{Proof of Theorem~\ref{ising_theorem}}\label{sec:proof_ising}
We first restate Theorem~\ref{ising_theorem} and then give the proof.
\begin{app-theorem}
Let $\D(A, \theta)$ be an unknown $n$-variable Ising model distribution with dependency graph $G$. Suppose that the $\D(A, \theta)$ has width $\lambda(A, \theta)\le \lambda$. Given $\rho\in(0,1)$ and $\eps>0$, if the number of i.i.d. samples satisfies $N = O(\lambda^2\exp(12\lambda)\ln(n/\rho)/\eps^4)$,
then with probability at least $1-\rho$, Algorithm~\ref{ising_learning_algo} produces $\hat{A}$ that satisfies
\begin{equation}
\max_{i,j\in[n]} |A_{ij} - \hat{A}_{ij}| \le \eps. 
\end{equation}
\end{app-theorem}

\begin{proof}
For ease of notation, we consider the $n$-th variable. The goal is to prove that Algorithm~\ref{ising_learning_algo} is able to recover the $n$-th row of the true weight matrix $A$. Specifically, we will show that if the number samples satisfies $N = O(\lambda^2\exp(O(\lambda))\ln(n/\rho)/\eps^4)$, then with probability as least $1-\rho/n$,
\begin{equation}
\max_{j\in[n]} |A_{nj} - \hat{A}_{nj}| \le \eps. \label{n_var}
\end{equation}
We then use a union bound to conclude that with probability as least $1-\rho$, $\max_{i,j\in[n]} |A_{ij} - \hat{A}_{ij}| \le \eps$.

Let $Z\sim\D(A, \theta)$, $X=[Z_{-n}, 1] = [Z_1, Z_2, \cdots, Z_{n-1},1]\in\{-1,1\}^n$, and $Y = Z_n \in\{-1,1\}$. By Fact~\ref{logistic_dist}, $\Pr[Y=1|X= x] = \sigma(\inprod{w^*, x})$, where $w^*= 2[A_{n1},\cdots, A_{n(n-1)}, \theta_n]$. Further, $\norm{w^*}_1\le 2\lambda$. Let $\hat{w}$ be the solution of the $\ell_1$-constrained logistic regression problem defined in (\ref{ERM}). 
%Our goal is to show that $|\hat{w}_j - w^*_j|\le 2\eps$ for $j\in[n-1]$, which implies that $|\hat{A}_{nj} - A_{nj}| \le \eps$ for $j\in[n-1]$.

By Lemma~\ref{KL_lemma}, if the number of samples satisfies $N=O(\lambda^2\ln(n/\rho)/\gamma^2)$, then with probability at least $1-\rho/n$, we have
\begin{equation}
\E_{X}[(\sigma(\inprod{w^*,X}) - \sigma(\inprod{\hat{w},X}))^2] \le \gamma. \label{close_l2}
\end{equation}

By Lemma~\ref{delta_unbias}, $Z_{-n}\in\{-1,1\}^{n-1}$ is $\delta$-unbiased (Definition~\ref{unbiased}) with $\delta = e^{-2\lambda}/2$. By Lemma~\ref{infty_norm}, if $\gamma<C_1\delta e^{-4\lambda} $ for some constant $C_1>0$, then (\ref{close_l2}) implies that
\begin{equation}
\norm{w^*_{1:(n-1)} - \hat{w}_{1:(n-1)}}_{\infty} \le O(1)\cdot e^{2\lambda}\cdot \sqrt{\gamma/\delta}. \label{close_infty}
\end{equation}
Note that $w^*_{1:(n-1)} = 2[A_{n1},\cdots, A_{n(n-1)}]$ and $\hat{w}_{1:(n-1)} = 2[\hat{A}_{n1},\cdots, \hat{A}_{n(n-1)}]$.  Let $\gamma= C_2 \delta e^{-4\lambda}\eps^2$ for some constant $C_2>0$ and $\eps\in (0,1)$, (\ref{close_infty}) then implies that 
\begin{equation}
\max_{j\in[n]} |A_{nj} - \hat{A}_{nj}| \le \eps.
\end{equation}
The number of samples needed is $N = O(\lambda^2\ln(n/\rho)/\gamma^2) =O(\lambda^2e^{12\lambda}\ln(n/\rho)/\eps^4)$.

We have proved that (\ref{n_var}) holds with probability at least $1-\rho/n$. Using a union bound over all $n$ variables gives that with probability as least $1-\rho$, $\max_{i,j\in[n]} |A_{ij} - \hat{A}_{ij}| \le \eps$.
\end{proof}

\section{Proof of Theorem~\ref{pairwise_theorem}}\label{sec:proof_pairwise}
The following lemma will be used in the proof.
\begin{lemma}
    Let $Z\sim \D$, where $\D$ is a $\delta$-unbiased distribution on $[k]^n$. Given $\alpha\neq \beta \in [k]$, conditioned on $Z_n \in \{\alpha, \beta\}$, $Z_{-n} \in [k]^{n-1}$ is also $\delta$-unbiased.
    \label{condition_unbiased}
\end{lemma}
\begin{proof}
For any $i\in [n-1]$, $a\in[k]$, and $x\in [k]^{n-2}$, we have
\begin{align}
    &\Pr[Z_i = a | Z_{[n]\backslash\{i,n\}}=x, Z_n \in \{\alpha, \beta\}] \nonumber\\
    &= \frac{\Pr[Z_i = a, Z_{[n]\backslash\{i,n\}}=x, Z_n = \alpha] + \Pr[Z_i = a, Z_{[n]\backslash\{i,n\}} = x, Z_n = \beta]}{\Pr[Z_{[n]\backslash\{i,n\}}=x, Z_n = \alpha]+ \Pr[Z_{[n]\backslash\{i,n\}}=x, Z_n = \beta]} \nonumber\\
    &\overset{(a)}{\ge} \min (\frac{\Pr[Z_i = a, Z_{[n]\backslash\{i,n\}}=x, Z_n = \alpha]}{\Pr[Z_{[n]\backslash\{i,n\}}=x, Z_n = \alpha]}, \frac{\Pr[Z_i = a, Z_{[n]\backslash\{i,n\}}=x, Z_n = \beta]}{\Pr[Z_{[n]\backslash\{i,n\}}=x, Z_n = \beta]}) \nonumber\\
    & = \min (\Pr[Z_i = a|Z_{[n]\backslash\{i,n\}}=x, Z_n = \alpha], \Pr[Z_i = a|Z_{[n]\backslash\{i,n\}}=x, Z_n = \beta])\nonumber\\
    & \overset{(b)}{\ge} \delta.
\end{align}
where (a) follows from the fact that $(a+b)/(c+d)\ge \min(a/c, b/d)$ for $a,b,c,d>0$, (b) follows from the fact that $Z$ is $\delta$-unbiased.
\end{proof}

Now we are ready to prove Theorem~\ref{pairwise_theorem}, which is restated below.
\begin{app-theorem}
Let $\D(\W, \Theta)$ be an $n$-variable pairwise graphical model distribution with width $\lambda(\W, \Theta)\le \lambda$ and alphabet size $k$. Given $\rho\in(0,1)$ and $\eps>0$, if the number of i.i.d. samples satisfies $N = O(\lambda^2k^4\exp(14\lambda)\ln(nk/\rho)/\eps^4)$, then with probability at least $1-\rho$, Algorithm~\ref{pairwise_learning_algo} produces $\hat{W}_{ij}\in \R^{k\times k}$ that satisfies
\begin{equation}
|W_{ij}(a,b) - \hat{W}_{ij}(a, b)| \le \eps, \quad \forall i \neq j \in[n], \; \forall a, b\in[k].
\end{equation}
\end{app-theorem}

\begin{proof}
To ease notation, let us consider the $n$-th variable (i.e., set $i=n$ inside the first ``for'' loop of Algorithm~\ref{pairwise_learning_algo}). The proof directly applies to other variables. We will prove the following result: if the number of samples $N = O(\lambda^2k^4\exp(14\lambda)\ln(nk/\rho)/\eps^4)$, then with probability at least $1-\rho/n$, the $U^{\alpha, \beta}\in \R^{n\times k}$ matrices produced by Algorithm~\ref{pairwise_learning_algo} satisfies
\begin{equation}
|W_{nj}(\alpha, b) - W_{nj}(\beta, b) - U^{\alpha, \beta}(j, b)| \le \eps, \quad \forall j\in[n-1],\; \forall \alpha, \beta, b\in[k]. \label{u_diff}
\end{equation}
Suppose that (\ref{u_diff}) holds, summing over $\beta\in [k]$ and using the fact that $\sum_{\beta} W_{nj}(\beta,b) = 0$ gives
\begin{equation}
|W_{nj}(\alpha, b)-\frac{1}{k}\sum_{\beta\in [k]}U^{\alpha, \beta}(j, b)| \le \eps, \quad \forall j\in[n-1],\; \forall \alpha, b\in [k].
\end{equation}
Theorem~\ref{pairwise_theorem} then follows by taking a union bound over the $n$ variables.

The only thing left is to prove (\ref{u_diff}). Now fix a pair of $\alpha, \beta \in [k]$, let $N^{\alpha, \beta}$ be the number of samples such that the $n$-th variable is either $\alpha$ or $\beta$. By Lemma~\ref{pairwise_KL_lemma} and Fact~\ref{pairwise_logistic_dist},  if $N^{\alpha, \beta} = O(\lambda^2 k \ln(n/\rho')/ \gamma^2)$, then with probability at least $1-\rho'$, the minimizer of the $\ell_{2,1}$ constrained logistic regression $w^{\alpha, \beta} \in \R^{n\times k}$ satisfies 
\begin{equation}
\E_{X}[(\sigma(\inprod{w^*,X}) - \sigma(\inprod{w^{\alpha, \beta},X}))^2] \le \gamma. \label{w_sigma_diff}
\end{equation}
Recall that $X\in \{0,1\}^{n\times k}$ is the one-hot encoding of the vector $[Z_{-n},1]\in [k]^n$, where $Z\sim\D(\W, \Theta)$ and $Z_n \in \{\alpha, \beta\}$. Besides, $w^*\in\R^{n\times k}$ satisfies
\begin{equation}
w^*(j,:) = W_{nj}(\alpha, :) - W_{nj}(\beta, :), \; \forall j\in[n-1];  \quad
w^*(n,:) = [\theta_n(\alpha)-\theta_n(\beta), 0,\cdots,0]. 
\end{equation} 
Let $U^{\alpha, \beta}\in \R^{n\times k}$ be formed by centering the first $n-1$ rows of $w^{\alpha, \beta}$. Since each row of $X$ is a standard basis vector (i.e., all 0's except a single 1),  $\inprod{U^{\alpha, \beta}, X} = \inprod{w^{\alpha, \beta}, X}$. Hence, (\ref{w_sigma_diff}) implies 
\begin{equation}
\E_{X}[(\sigma(\inprod{w^*,X}) - \sigma(\inprod{U^{\alpha, \beta},X}))^2] \le \gamma. \label{u_sigma_diff}
\end{equation}

By Lemma~\ref{delta_unbias}, we know that $Z\sim\D(\W, \Theta)$ is $\delta$-unbiased with $\delta=e^{-2\lambda}/k$. By Lemma~\ref{condition_unbiased}, conditioned on $Z_n \in \{\alpha, \beta\}$, $Z_{-n}$ is also $\delta$-unbiased. Hence, the condition of Lemma~\ref{pairwise_infty_norm} holds.
Applying Lemma~\ref{pairwise_infty_norm} to (\ref{u_sigma_diff}), we get that if $N^{\alpha, \beta} = O(\lambda^2k^3 \exp(12\lambda) \ln(n/\rho'))/\eps^4)$, the following holds with probability at least $1-\rho'$:
\begin{equation}
|W_{nj}(\alpha, b) - W_{nj}(\beta, b)-U^{\alpha, \beta}(j, b)| \le \eps, \;\forall j\in[n-1], \; \forall b\in[k]. \label{u_diff_pair}
\end{equation}

So far we have proved that (\ref{u_diff}) holds for a fixed $(\alpha, \beta)$ pair. This requires that $N^{\alpha, \beta}= O(\lambda^2k^3 \exp(12\lambda) \ln(n/\rho'))/\eps^4)$. Recall that $N^{\alpha, \beta}$ is the number of samples that the $n$-th variable takes $\alpha$ or $\beta$. We next derive the number of total samples needed in order to have $N^{\alpha, \beta}$ samples for a given $(\alpha, \beta)$ pair. Since $\D(\W, \Theta)$ is $\delta$-unbiased with $\delta=e^{-2\lambda(\W, \Theta)}/k$, for $Z\sim\D(\W, \Theta)$, we have $\Pr[Z_{n}\in \{\alpha, \beta\}| Z_{-n}] \ge 2\delta$, and hence $\Pr[Z_{n}\in \{\alpha, \beta\}]\ge 2\delta$. By the Chernoff bound, if the total number of samples satisfies $N =O(N^{\alpha, \beta}/\delta+\log(1/\rho'')/\delta)$, then with probability at least $1-\rho''$, we have $N^{\alpha, \beta}$ samples for a given $(\alpha, \beta)$ pair.

To ensure that (\ref{u_diff_pair}) holds for all $(\alpha, \beta)$ pairs with probability at least $1-\rho/n$, we can set $\rho' = \rho/(nk^2)$ and $\rho'' = \rho/(nk^2)$ and take a union bound over all $(\alpha, \beta)$ pairs. The total number of samples required is $N = O(\lambda^2k^4\exp(14\lambda)\ln(nk/\rho)/\eps^4)$.

We have shown that (\ref{u_diff}) holds for the $n$-th variable with probability at least $1-\rho/n$. By the discussion at the beginning of the proof, Theorem~\ref{pairwise_theorem} then follows by a union bound over the $n$ variables.
\end{proof}

\section{Mirror descent algorithms for constrained logistic regression} \label{sec:mirror}

Algorithm~\ref{l1_mirror_descent} gives a mirror descent algorithm for the following $\ell_1$-constrained logistic regression:
\begin{equation}
\min_{w\in\R^n} \frac{1}{N} \sum_{i=1}^N \ln(1+e^{-y^i\inprod{w,x^i}}) \quad\text{\quad s.t. }\norm{w}_1\le W_1. \label{l1_ERM}
\end{equation}
We use the doubling trick to expand the dimension and re-scale the samples (Step 2 in Algorithm~\ref{l1_mirror_descent}). Now the original problem becomes a logistic regression problem over a probability simplex: $\Delta_{2n+1}=\{w\in\R^{2n+1}: \sum_{i=1}^{2n+1} w_i=1, w_i\ge0, \forall i\in[2n+1]\}$.
\begin{equation}
\min_{w\in\Delta_{2n+1}} 
\frac{1}{N} \sum_{i=1}^N -\hat{y}^i\ln(\sigma(\inprod{w,\hat{x}^i})) -(1-\hat{y}^i)\ln(1-\sigma(\inprod{w,\hat{x}^i})), 
\end{equation}
where $(\hat{x}^i, \hat{y}^i)\in\R^{2n+1} \times \{0,1\}$. In Step 4-11 of Algorithm~\ref{l1_mirror_descent}, we follow the standard simplex setup for mirror descent algorithm (see Section 5.3.3.2 of~\citep{BN13}). Specifically, the negative entropy is used as the distance generating function (aka the mirror map). The projection step (Step 9) can be done by a simple $\ell_1$ normalization operation. After that, we transform the solution back to the original space (Step 12).

\begin{algorithm2e}
\caption{Mirror descent algorithm for $\ell_1$-constrained logistic regression}
\label{l1_mirror_descent}
\DontPrintSemicolon
\LinesNumbered
\KwIn{$\{(x^i, y^i)\}_{i=1}^N$ where $x^i\in\{-1,1\}^n$, $y^i\in\{-1,1\}$; constraint on the $\ell_1$ norm $W_1\in\R_{+}$; number of iterations $T$.}
\KwOut{$\bar{w}\in \R^n$.}
\For{sample $i\leftarrow 1$ \KwTo $N$}{
 \tcp{Form samples $(\hat{x}^i, \hat{y}^i)\in\R^{2n+1} \times \{0,1\}$.}
 $\hat{x}^i \leftarrow [x^i, -x^i, 0]\cdot W_1$, \quad $\hat{y}^i \leftarrow (y^i+1)/2$\;
}
\tcp{Initialize $w$ as the uniform distribution.}
$w^1 \leftarrow [\frac{1}{2n+1},\frac{1}{2n+1},\cdots,\frac{1}{2n+1}]\in\R^{2n+1}$\;
$\gamma \leftarrow \frac{1}{2W_1}\sqrt{\frac{2\ln(2n+1)}{T}}$ \tcp*{Set the step size.}
\For{iteration $t\leftarrow 1$ \KwTo $T$}{
 $g^t \leftarrow \frac{1}{N}\sum_{i=1}^N (\sigma(\inprod{w^t, \hat{x}^i})-\hat{y}^i)\hat{x}^i$ \tcp*{Compute the gradient.}
 $w^{t+1}_i \leftarrow w_i^t \exp(-\gamma g^t_i) $, for $i\in[2n+1]$ \tcp*{Coordinate-wise update.}
 $w^{t+1} \leftarrow w^{t+1}/\norm{w^{t+1}}_1$ \tcp*{Projection step.}
}
$\bar{w} \leftarrow \sum_{t=1}^{T} w^t/T$ \tcp*{Aggregate the updates.}
\tcp{Transform $\bar{w}$ back to $\R^n$ and the actual scale.}
$\bar{w} \leftarrow (\bar{w}_{1:n} -  \bar{w}_{(n+1):2n})\cdot W_1$ 
\end{algorithm2e}

Algorithm~\ref{l21_mirror_descent} gives a mirror descent algorithm for the $\ell_{2,1}$-constrained logistic regression:
\begin{equation}
\min_{w\in\R^{n \times k}} \frac{1}{N} \sum_{i=1}^N \ln(1+e^{-y^i\inprod{w,x^i}}) \quad\text{\quad s.t. }\norm{w}_{2,1}\le W_{2,1}. \label{l21_ERM}
\end{equation}
For simplicity, we assume that $n\ge 3$\footnote{For $n\le 2$, we need to switch to a different mirror map, see Section 5.3.3.3 of~\citep{BN13} for more details.}. We then follow Section 5.3.3.3 of~\citep{BN13} to use the following function as the mirror map $\Phi: \R^{n\times k} \to \R$:
\begin{equation}
\Phi(w) = \frac{e\ln(n)}{p} \norm{w}_{2, p}^p,\quad p = 1+1/\ln(n). \label{Phi_def}
\end{equation}
The update step (Step 8) can be computed efficiently in $O(nk)$ time, see the discussion in Section 5.3.3.3 of~\citep{BN13} for more details.

\begin{algorithm2e}
\caption{Mirror descent algorithm for $\ell_{2,1}$-constrained logistic regression}
\label{l21_mirror_descent}
\DontPrintSemicolon
\LinesNumbered
\KwIn{$\{(x^i, y^i)\}_{i=1}^N$ where $x^i\in\{0,1\}^{n\times k}$, $y^i\in\{-1,1\}$; constraint on the $\ell_{2,1}$ norm $W_{2,1}\in\R_{+}$; number of iterations $T$.}
\KwOut{$\bar{w}\in \R^{n\times k}$.}
\For{sample $i\leftarrow 1$ \KwTo $N$}{
 \tcp{Form samples $(\hat{x}^i, \hat{y}^i)\in\R^{n \times k} \times \{0,1\}$.}
 $\hat{x}^i \leftarrow x^i \cdot W_{2,1}$, \quad $\hat{y}^i \leftarrow (y^i+1)/2$
}
\tcp{Initialize $w$ as a constant matrix.}
$w^1 \leftarrow [\frac{1}{n\sqrt{k}},\frac{1}{n\sqrt{k}},\cdots,\frac{1}{n\sqrt{k}}]  \in\R^{n \times k}$\;
$\gamma \leftarrow \frac{1}{2W_{2,1}}\sqrt{\frac{e\ln(n)}{T}}$ \tcp*{Set the step size.}
\For{iteration $t\leftarrow 1$ \KwTo $T$}{
 $g^t \leftarrow \frac{1}{N}\sum_{i=1}^N (\sigma(\inprod{w^t, \hat{x}^i})-\hat{y}^i)\hat{x}^i$ \tcp*{Compute the gradient.}
 $w^{t+1} \leftarrow \arg\min_{\norm{w}_{2,1}\le 1} \Phi(w) - \inprod{\nabla \Phi(w^{t}) - \gamma g^t, w}$\tcp*{$\Phi(w)$ is in (\ref{Phi_def}).}
}
$\bar{w} \leftarrow (\sum_{t=1}^{T} w^t/T)\cdot W_{21}$ \tcp*{Aggregate the updates.}
\end{algorithm2e}

\section{Proof of Theorem~\ref{ising_runtime} and Theorem~\ref{pairwise_runtime}}\label{sec:runtime_proof}
\begin{lemma}
Let $\hat{\Loss}(w)=\frac{1}{N} \sum_{i=1}^N \ln(1+e^{-y^i\inprod{w,x^i}})$ be the empirical loss. Let $\hat{w}$ be a minimizer of the ERM defined in (\ref{l1_ERM}).
The output $\bar{w}$ of Algorithm~\ref{l1_mirror_descent} satisfies
\begin{equation}
\hat{\Loss}(\bar{w}) - \hat{\Loss}(\hat{w}) \le 2W_1\sqrt{\frac{2\ln(2n+1)}{T}}. \label{l1_conv_T}
\end{equation}
Similarly, let $\hat{w}$ be a minimizer of the ERM defined in (\ref{l21_ERM}). Then the output $\bar{w}$ of Algorithm~\ref{l21_mirror_descent} satisfies
\begin{equation}
\hat{\Loss}(\bar{w}) - \hat{\Loss}(\hat{w}) \le O(1)\cdot W_{2,1}\sqrt{\frac{\ln(n)}{T}}. \label{l21_conv_T}
\end{equation}
\label{convergence}
\end{lemma}
Lemma~\ref{convergence} follows from the standard convergence result for mirror descent algorithm (see, e.g., Theorem 4.2 of~\citep{Bub15}), and the fact that the gradient $g^t$ in Step 6 of Algorithm~\ref{l1_mirror_descent} satisfies $\norm{g^t}_{\infty}\le 2W_1$ (reps. the gradient $g^t$ in Step 6 of Algorithm~\ref{l21_mirror_descent} satisfies $\norm{g^t}_{\infty}\le 2W_{2,1}$). This implies that the objective function after rescaling the samples is $2W_1$-Lipschitz w.r.t. $\norm{\cdot}_1$ (reps. $2W_{2,1}$-Lipschitz w.r.t. $\norm{\cdot}_{2,1}$).

We are now ready to prove Theorem~\ref{ising_runtime}, which is restated below.
\begin{app-theorem}
In the setup of Theorem~\ref{ising_theorem}, suppose that the $\ell_1$-constrained logistic regression in Algorithm~\ref{ising_learning_algo} is optimized by the mirror descent method (Algorithm~\ref{l1_mirror_descent}) given in Appendix~\ref{sec:mirror}. Given $\rho\in(0,1)$ and $\eps>0$, if the number of mirror descent iterations satisfies $T=O(\lambda^2\exp(12\lambda)\ln(n)/\eps^4$, and the number of i.i.d. samples satisfies $N = O(\lambda^2\exp(12\lambda)\ln(n/\rho)/\eps^4)$, then (\ref{close_A}) still holds with probability at least $1-\rho$. The total run-time of Algorithm~\ref{ising_learning_algo} is $O(TNn^2)$.
\end{app-theorem}

\begin{proof}
We first note that in the proof of Theorem~\ref{ising_theorem}, we only use $\hat{w}$ in order to apply the result from Lemma~\ref{KL_lemma}. In the proof of Lemma~\ref{KL_lemma} (given in Appendix~\ref{sec:proof_kl_lemma}), there is only one place where we use the definition of $\hat{w}$: the inequality (b) in (\ref{hat_bound}). As a result, if we can show that (\ref{hat_bound}) still holds after replacing $\hat{w}$ by $\bar{w}$, i.e.,
\begin{equation}
\Loss(\bar{w}) \le \Loss(w^*) + O(\gamma), \label{new_bound}
\end{equation}
then Lemma~\ref{KL_lemma} would still hold, and so is Theorem~\ref{ising_theorem}.

By Lemma~\ref{convergence}, if the number of iterations satisfies $T = O(W_1^2 \ln(n)/\gamma^2)$, then
\begin{equation}
\hat{\Loss}(\bar{w}) - \hat{\Loss}(\hat{w}) \le \gamma. \label{conv_bound}
\end{equation}
As a result, we have
\begin{equation}
\Loss(\bar{w}) \overset{(a)}{\le} \hat{\Loss}(\bar{w}) +  \gamma
\overset{(b)}{\le} \hat{\Loss}(\hat{w}) +  2\gamma
\overset{(c)}{\le} \hat{\Loss}(w^*) +  2\gamma
\overset{(d)}{\le} \Loss(w^*) + 3\gamma, \label{w_bar_diff}
\end{equation} 
where (a) follows from (\ref{eps_bound}), (b) follows from (\ref{conv_bound}), (c) follows from the fact that $\hat{w}$ is the minimizer of $\hat{\Loss}(w)$, and (d) follows from (\ref{con_bound}). The number of mirror descent iterations needed for (\ref{new_bound}) to hold is $T = O(W_1^2 \ln(n)/\gamma^2)$. In the proof of Theorem~\ref{ising_theorem}, we need to set $\gamma = O(1)\eps^2 \exp(-6\lambda)$ (see the proof following (\ref{close_infty})), so the number of mirror descent iterations needed is $T=O(\lambda^2\exp(12\lambda)\ln(n)/\eps^4)$.

To analyze the runtime of Algorithm~\ref{ising_learning_algo}, note that for {\em each variable} in $[n]$, transforming the samples takes $O(N)$ time, solving the $\ell_1$-constrained logisitic regression via Algorithm~\ref{l1_mirror_descent} takes $O(TNn)$ time, and updating the edge weight estimate takes $O(n)$ time. Forming the graph $\hat{G}$ over $n$ nodes takes $O(n^2)$ time. The total runtime is $O(TNn^2)$.
\end{proof}

The proof of Theorem~\ref{pairwise_runtime} is identical to that of Theorem~\ref{ising_runtime} and is omitted here. The key step is to show that (\ref{new_bound}) holds after replacing $\hat{w}$ by $\bar{w}$. This can be done by using the convergence result in Lemma~\ref{convergence} and applying the same logic in (\ref{w_bar_diff}). The runtime of Algorithm~\ref{pairwise_learning_algo} can be analyzed in the same way as above. The $\ell_{2,1}$-constrained logistic regression dominates the total runtime. It requires $O(TN^{\alpha, \beta}nk)$ time for each pair $(\alpha, \beta)$ and each variable in $[n]$, where $N^{\alpha, \beta}$ is the subset of samples that a given variable takes either $\alpha$ or $\beta$. Since $N\ge kN^{\alpha, \beta}$, the total runtime is $O(TNn^2k^2)$.

\section{More experimental results}\label{sec:more_exp}
We compare our algorithm (Algorithm~\ref{pairwise_learning_algo}) with the Sparsitron algorithm in~\citep{KM17} on a two-dimensional 3-by-3 grid (shown in Figure~\ref{grid_graph}). We experiment three alphabet sizes: $k=2,4,6$. For each value of $k$, we simulate both algorithms 100 runs, and in each run we generate the $W_{ij}$ matrices with entries $\pm 0.2$. To ensure that each row (as well as each column) of $W_{ij}$ is centered (i.e., zero mean), we will randomly choose $W_{ij}$ between two options: as an example of $k=2$, $W_{ij} = [0.2, -0.2; -0.2, 0.2]$ or $W_{ij} = [-0.2, 0.2; 0.2, -0.2]$. The external field is zero. Sampling is done via exactly computing the distribution. The Sparsitron algorithm requires two sets of samples: 1) to learn a set of candidate weights; 2) to select the best candidate. We use $\max\{200, 0.01\cdot N\}$ samples for the second part. We plot the estimation error $\max_{ij} \norm{W_{ij} - \hat{W}_{ij}}_{\infty}$ and the fraction of successful runs (i.e., runs that exactly recover the graph) in Figure~\ref{app_non_binary_comp}. Compared to the Sparsitron algorithm, our algorithm requires fewer samples for successfully recovery. 

\begin{figure}[ht]
\centering
\includegraphics[width=0.8\textwidth]{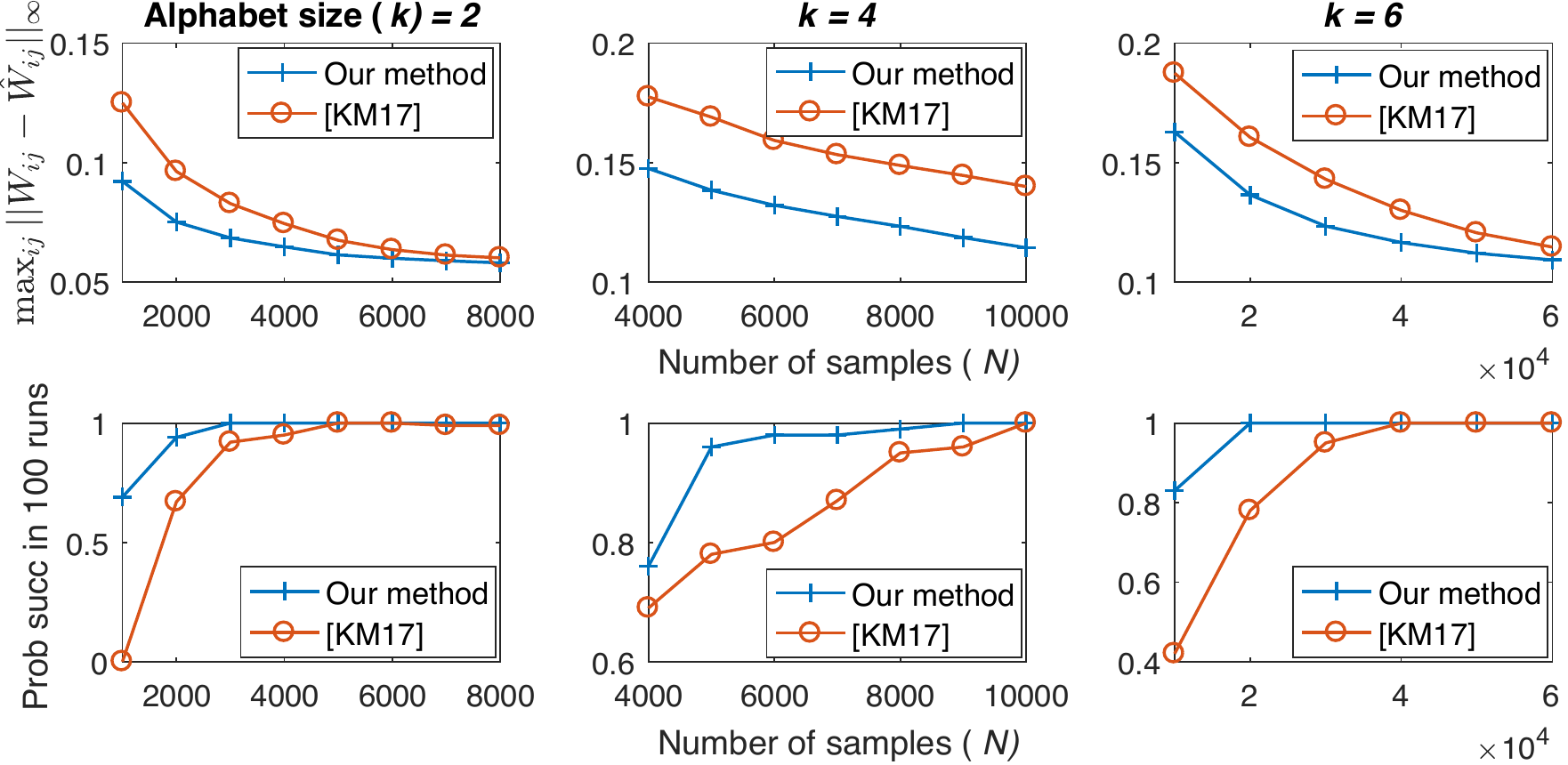}
\caption{Comparison of our algorithm and the Sparsitron algorithm in~\citep{KM17} on a two-dimensional 3-by-3 grid. Top row shows the average of the estimation error $\max_{ij} \norm{W_{ij} - \hat{W}_{ij}}_{\infty}$. Bottom row plots the faction of successful runs (i.e., runs that exactly recover the graph). Each column corresponds to an alphabet size: $k=2,4,6$. Our algorithm needs fewer samples than the Sparsitron algorithm for graph recovery.}
\label{app_non_binary_comp}
\end{figure}

\end{document}